\newcommand{\titleText}{T-ARC: Topology-Aware Randomized Clustering \\ via Distributionally Robust Stochastic Block Models}

\newcommand{\HeadingText}{T-ARC}

\documentclass[11pt]{article}
\usepackage{geometry}
\usepackage[T1]{fontenc}
\usepackage[utf8]{inputenc}
\usepackage{graphicx}
\usepackage{amssymb,amsthm,mathtools}
\usepackage{hyperref}

\usepackage{enumitem,multirow,array}
\usepackage{hhline}
\usepackage{booktabs}
\usepackage{adjustbox}
\usepackage{subcaption}

\usepackage[linesnumbered,ruled,vlined,noline]{algorithm2e}
\SetKwInput{KwInput}{Input}
\SetKwInput{KwOutput}{Output}

\usepackage{comment,soul}
\usepackage[normalem]{ulem}
\usepackage{cleveref}
\crefname{section}{§}{§§}
\Crefname{section}{§}{§§}

\usepackage{bm}                 
\DeclareMathOperator*{\argmin}{\textrm{argmin}\,} 
\def\diag            {\textrm{diag}}     
\def\Diag            {\textrm{Diag}} 	 
\def\ones            {\pmb{1}}            
\def\st              {\,\textrm{s.t.}\,} 
\def\tr              {\textrm{Tr}}       
\def\IE {\mathbb{E}} 		
\def\IR  {\mathbb{R}} 	    
\def\IRnn{\IR^{n \times n}} 

\def\cB { \mathcal{B}}

\def\cG { \mathcal{G}}

\def\cJ { \mathcal{J}}
 
\def\cL { \mathcal{L}}

\def\cX { \mathcal{X}}

\newtheorem{theorem}{Theorem}

\newtheorem*{observation}{Observation}

\usepackage{fancyhdr}

\usepackage{authblk}
\author[1]{Serena Grazia De Benedictis \thanks{Corresponding author: \texttt{serena.debenedictis@uniba.it}}}
\author[2]{Andersen Ang}
\author[1]{Nicoletta Del Buono}
\author[1]{Flavia Esposito}
\author[3]{Laura Selicato}
\affil[1]{Department of Mathematics, University of Bari Aldo Moro, Italy}
\affil[2]{Electronics and Computer Science Department, University of Southampton, UK}
\affil[3]{National Research Council (CNR), Water Research Institute (IRSA), Italy}

\title{\Large \titleText \vspace{-2ex}}
\date{ \vspace{-3ex}}
\begin{document}\maketitle

\newcommand{\sdc}[1]{(\textbf{Serena's comment:} #1)}
\newcommand{\sdi}[1]{#1}

\newcommand{\aac}[1]{(\textbf{Andersen's comment:} #1)}
\newcommand{\aai}[1]{#1}

\newcommand{\ndbc}[1]{(\textbf{Nicoletta's comment:} #1)}
\newcommand{\ndbi}[1]{#1}

\newcommand{\fec}[1]{(\textbf{Flavia's comment:} #1)}
\newcommand{\fei}[1]{#1}

\newcommand{\lsc}[1]{(\textbf{Laura's comment:} #1)}
\newcommand{\lsi}[1]{#1}

\begin{abstract}
In this work, we introduce a new clustering method, namely \textbf{T-ARC} (Topology-Aware Randomized Clustering), that corrects the geometric bias of K-means by embedding topological information directly into the optimization objective. Building on the assumption that the data admits an underlying hidden structure modeled via a latent graph, the idea is to uncover this information through the interplay between the standard K-means data-fidelity term and a graph-cut penalty, which discourages cluster assignments inconsistent with the connectivity structure of the data.

To render this coupling tractable, the latent graph is modeled as a random realization from a Stochastic Block Model (SBM), whose scalar parameter is optimized within a Distributionally Robust Optimization (DRO) framework, yielding a closed-form proximal update. Both SBM and DRO are informed by a persistence-based similarity matrix derived from zero-dimensional persistent homology ($H_0$), which translates the multiscale connectivity structure of the data into a pairwise topological prior.
The overall optimization proceeds via Block Coordinate Descent; convergence is established through a global Lyapunov functional: the deterministic blocks satisfy monotonic descent, while the stochastic graph update satisfies descent in expectation, so that the expected energy converges.

Experiments on synthetic datasets with non-convex geometries and on random subsets of Fashion-MNIST show that T-ARC recovers latent topological structures where K-means fails, achieving the highest accuracy on curved and interleaved clusters while remaining competitive, and markedly more stable than K-means, on real data.
\end{abstract}

\tableofcontents


\section{Introduction}

In the big data era, the ability to extract meaningful patterns from complex, high-dimensional datasets is a cornerstone of data analysis. Unsupervised learning (clustering in particular) plays a pivotal role by uncovering hidden structures within data without the need for labeled examples. Among clustering algorithms, K-means \cite{lloyd1982,jain2010} remains a widely adopted method, owing to its simplicity, scalability, and broad empirical effectiveness. However, K-means and most of its variants rely fundamentally on Euclidean distances and implicitly assume that clusters are convex and roughly spherical in shape. This geometric bias limits their ability to recover clusters that are elongated, curved, or interleaved: configurations that arise routinely in real-world high-dimensional data \cite{kmeans_limitations}.

This motivates a shift from a purely geometric approach toward a topological perspective in which the connectivity structure of the data becomes the primary object of analysis. Topology offers a natural language to formalize this intuition: topological methods analyze the \emph{shape} of data by encoding structural properties such as connectedness and multi-scale organization, features that are invariant under continuous deformations and thus robust to data geometric distortions \cite{SelMATCOM}. Incorporating such structural information directly into a clustering objective is thus a promising direction to overcome the limitations of purely metric-based methods \cite{carlsson2009,vonluxburg2007}.

We propose a new clustering method called \textbf{T}opology-\textbf{A}ware \textbf{R}andomized \textbf{C}lustering (\textbf{T-ARC}) that pursues precisely this direction. The key idea of T-ARC is to assume that the data admits a latent graph structure, and to complement the standard K-means data-fidelity term with a graph-cut term \cite{huang2025inhomogeneous}, penalizing cluster assignments inconsistent with the connectivity structure of the data. This coupling effectively reveals this latent structure by promoting cluster boundaries that respect the topological organization of the data, rather than its metric arrangement alone, thereby correcting the geometric bias of K-means.

Most existing approaches that incorporate graph information into clustering rely on a similarity graph estimated once from the data and held fixed throughout the optimization \cite{cai2011graph,vonluxburg2007}. T-ARC departs from this paradigm by treating the latent graph itself as a variable: rather than regularizing against a fixed Laplacian, we model it as a random realization from a parametric family and optimize its distribution jointly with the cluster assignments, under a distributionally robust criterion that accounts for uncertainty in the graph itself.

Formalizing this idea in a tractable way is non-trivial. The connectivity graph over the data points is not known \emph{a priori}, and the space of all admissible graphs is combinatorially intractable, due to its dimension. To address this, we model the latent graph as a random realization from a parametric Stochastic Block Model (SBM) \cite{SBM}, whose parameters are optimized within a Distributionally Robust Optimization (DRO) framework \cite{DRO}. This formulation yields closed-form update rules, solved efficiently via a proximal point method \cite{proximal_point_method}, and ensures robustness to uncertainty in the graph distribution.

Central to the framework is a \emph{persistence-based similarity matrix}~$\bm{S}$, constructed from the zero-dimensional persistent homology of the data~\cite{edelsbrunner2010,zomorodian2005}. It plays a twofold role: parametrizing the SBM edge-probability matrix and defining the center of the DRO ambiguity set, acting in both cases as a structural prior on the latent graph distribution. Unlike conventional similarity measures, which capture only local, pairwise relationships, $\bm{S}$ encodes the multiscale connectivity structure of the data --- provably stable with respect to input perturbations \cite{cohen-steiner2007} --- and propagates it, through the graph-cut term, directly into the clustering objective. This approach instantiates the paradigm of \textit{Topological Machine Learning} (TML) \cite{hensel2021} in the context of clustering, where topological information enters not as an auxiliary signal but as a structural prior that shapes the optimization itself.

The overall optimization proceeds via Block Coordinate Descent (BCD) \cite{BCD}, alternating between updates of the cluster assignment matrix, the centroids, and the graph parameters. While the updates for the cluster assignments and centroids enjoy standard convexity properties, the stochastic nature of the graph update requires a dedicated analysis. We provide a convergence analysis based on a global Lyapunov functional: every block update, including the proximal update of the stochastic graph parameter, decreases its expected value, which is therefore convergent, while the successive differences of all blocks vanish.
In summary, this work makes the following contributions:
\begin{itemize}[leftmargin=*]\setlength{\itemsep}{0pt}
  \item A clustering framework (T-ARC) that corrects the geometric bias of K-means by embedding topological information into the optimization objective via a graph-cut term.
  \item A persistence-based similarity matrix grounded in the zero-dimensional persistent homology, serving as a structural prior that anchors both the SBM parametrization and the DRO ambiguity set to the multiscale connectivity structure of the data.
  \item A tractable optimization strategy that jointly learns the graph structure and cluster assignments by reformulating graph learning as DRO over SBM, leading to closed-form updates and convergence guarantees in expectation via a global Lyapunov analysis.
\end{itemize}
Finally, we validate T-ARC on synthetic datasets with complex geometric and topological structures and on random subsets of Fashion-MNIST \cite{fashion_MNIST_dataset}, demonstrating competitive performance against established baselines.

The remainder of this paper is organized as follows.
\Cref{sec:formulation} presents the mathematical formulation of the T-ARC framework, detailing the integration of the graph-cut term with the K-means objective. 
\Cref{sec:optimization_strategy} describes the block coordinate descent optimization strategy, including the closed-form updates for the cluster assignments, centroids, and the stochastic graph parameter via the proximal point method. 
\Cref{sec:convergence} provides the convergence analysis, establishing the descent in expectation of a global Lyapunov functional and the convergence of its expected value. 
\Cref{sec:persistence_based_similarity_matrix} defines the construction of the persistence-based similarity matrix derived from zero-dimensional persistent homology, which serves as the structural prior for the model. 
\Cref{sec:implementation} presents and discusses the experimental results on both synthetic datasets with complex topologies and random subsets of the Fashion-MNIST dataset. 
Finally, \Cref{sec:conclusion} concludes the paper and outlines promising directions for future work. 
The feasibility constraints on the Stochastic Block Model parameter $b$ are derived in Appendix~\ref{sec:appendix_constraints}.

\section{Mathematical Problem Formulation}\label{sec:formulation}

Let $\bm{X} = [\bm{x}_1,\dots,\bm{x}_n]^\top \in \IR^{n \times d}$ be the data matrix, where each row $\bm{x}_i \in \IR^d$ represents a data point, and $n$ is the number of samples. Our goal is to partition the dataset $\{\bm{x}_1,\dots,\bm{x}_n\}$ into $k$ disjoint clusters.

A standard approach is to solve the K-means optimization problem \cite{ding2010convex}, but this formulation is purely geometric: it assumes convex, spherical clusters and relies on local, pairwise information, remaining blind to the global connectivity structure of the data. This limits its applicability when clusters are elongated, curved, or interconnected, scenarios where topological information becomes crucial.

To overcome this limitation, we enrich the geometric perspective with a topological one. Specifically, we assume the dataset is endowed with an underlying latent topological structure, which we approximate by an undirected and unweighted graph $G = (V,E)$. In $G$, each node corresponds to a data point $\bm{x}_i$, and the edge set $E$ encodes adjacency or similarity relations. In order to find this underlying latent topological structure, our approach treats the graph itself as a variable: we aim to discover, among all possible graphs on the $n$ vertices---a set that we will denote with $\cG$---the specific configuration that best supports the clustering objective. We thus augment the K-means objective with a graph-cut term, obtaining the following combined regularized optimization problem:
\begin{equation}\label{eq:initial_opt_problem}
\min_{\bm{C},\bm{\mu},\bm{L}}  F(\bm{C},\bm{\mu},\bm{L}) =
\underbrace{\frac{1}{2}\tr\!\big(\bm{C}^\top \bm{L} \bm{C}\big)}_{\text{Graph-cut}}
+
\underbrace{\frac{\lambda_\mu}{2}\| \bm{X} - \bm{C}\bm{\mu} \|_F^2}_{\text{K-means}}
\quad\st\quad
\bm{C} \in \IR_+^{n \times k}, \bm{\mu} \in \IR_+^{k \times d}, \bm{L} \in \cL_{\cG},
\end{equation}
where $\lambda_\mu > 0$ balances the two contributions and 
\[
\cL_{\cG} = \left\{ \bm{L} = \bm{D} - \bm{A} 
~\middle|
\begin{array}{l}
\bm{A} \in \{0,1\}^{n \times n}, \bm{A} = \bm{A}^\top, \diag(\bm{A})=0 \ \text{adjacency matrix of} \ G \\
\bm{D} = \diag(\bm{A} \ones_n) \ \text{degree matrix of} \ G \ \text{and}\\
G = (V,E) \in \cG
\end{array}
\right\}
\]
denotes the set of Laplacians representing all such graphs.
The two terms play complementary roles, reflecting the two perspectives we aim to balance:
\begin{itemize}[leftmargin=*]\setlength{\itemsep}{0pt}
  \item \textbf{Graph-cut term:} This is the topological driver of the model. By coupling the cluster assignment matrix $\bm{C}$ with the graph Laplacian $\bm{L}$, it penalizes assignments that cut across the connectivity structure of $G$. It biases the solution toward clusters that respect the underlying topological organization.
  \item \textbf{K-means term:} This acts as a geometric regularizer. It ensures that clusters remain compact in the original feature space, preventing the graph-cut term from producing degenerate or overly fragmented partitions. It accounts for intra-cluster variance without explicitly using graph information, thereby balancing topological coherence with geometric fidelity.
\end{itemize}
In essence, the geometric perspective (data as points in space) enters in the clustering task through the K-means term, while the topological perspective (a latent graph on the data) enters through the graph-cut term. Their interplay, modulated by $\lambda_\mu$, defines a trade-off where the optimal clustering must simultaneously minimize intra-cluster variance in the ambient space and respect the connectivity structure encoded in the graph.

\section{Optimization Strategy} \label{sec:optimization_strategy}

The convexity properties of the objective function with respect to each variable motivate the optimization strategy adopted in this work. In particular:
\begin{itemize}[leftmargin=*]\setlength{\itemsep}{0pt}
  \item The graph-cut term is a quadratic form in $\bm{C}$. As the graph Laplacian $\bm{L}$ is positive semi-definite (PSD), this term is convex in $\bm{C}$.

  \item The K-means term is non-convex in the joint pair $(\bm{C},\bm{\mu})$ due to the bilinear structure of $\bm{C}\bm{\mu}$. However, it is block-convex in each variable separately.

  \item Regarding the constraints, the non-negativity sets $\bm{C} \geq 0$ and $\bm{\mu} \geq 0$ are convex. The set $\cL_{\cG}$ of graph Laplacians is non-convex.
\end{itemize}
Thus, the full optimization problem is \emph{jointly non-convex}, due to both the bilinear term $\bm{C}\bm{\mu}$ and the non-convex set $\cL_{\cG}$. The problem is block-wise convex, meaning that when all other variables are held fixed, the optimization over each remaining variable is convex, so it is natural to use BCD, in which the variables are updated sequentially:
\begin{itemize}[leftmargin=*]\setlength{\itemsep}{0pt}
    \item \textbf{Update $\bm C$:} $\bm{C}_{k+1} = \underset{\bm{C} \geq 0}{\argmin} \ F(\bm{C}, \bm{\mu}_k, \bm{L}_k)$.
    \item \textbf{Update $ \bm\mu$:} $\bm{\mu}_{k+1} = \underset{\bm{\mu} \geq 0}{\argmin} \ F(\bm{C}_{k+1}, \bm{\mu}, \bm{L}_k)$.
    \item \textbf{Update $\bm L$:} $\bm{L}_{k+1} = \underset{\bm{L} \in \cL_{\cG}}{\argmin} \ F(\bm{C}_{k+1}, \bm{\mu}_{k+1}, \bm{L})$.
\end{itemize}
In the following, we will write $F(\bm{C})$, $F(\bm{\mu})$, and $F(\bm{L})$ to denote the objective function restricted to each block of variables, with the remaining blocks held fixed at their current values.

\subsection{Update of $C$ and $\mu$} \label{subsec:update_C_mu}
For fixed $\bm{\mu}$ and $\bm{L}$, the optimization of $\bm{C}$ reduces to a regularized Non-Negative Least Squares (NNLS):
\[
  \min_{\bm{C} \geq 0} F(\bm{C})  
  = 
  \frac{1}{2}\tr(\bm{C}^\top \bm{L} \bm{C}) 
  + \frac{\lambda_\mu}{2} \|\bm{X} - \bm{C}\bm{\mu}\|_F^2.
\]
This subproblem is solved by the projected gradient step $\bm{C}_{k+1} = [\bm{C}_k - \alpha \nabla_{\bm{C}} F(\bm{C})]_+$, where the projection $[\,\cdot\,]_+$ is applied elementwise and denotes projection onto the nonnegative orthant. The gradient $\nabla_{\bm{C}} F(\bm{C}) = \bm{L}\bm{C} + \lambda_\mu(\bm{C}\bm{\mu}\bm{\mu}^\top - \bm{X}\bm{\mu}^\top)$ has Lipschitz constant $\|\bm{L}\|_2 + \lambda_\mu\|\bm{\mu}\bm{\mu}^\top\|_2$. This step is iterated, forming an inner loop,  which is terminated when either the maximum number of inner iterations is reached, or the relative change of the objective $F$ between two consecutive inner iterations falls below a prescribed tolerance.

When $\bm{C}$ is fixed, the optimization over $\bm{\mu}$ reduces to the NNLS $\min_{\bm{\mu} \geq 0} F(\bm{\mu}) = \tfrac{1}{2}\|\bm{X} - \bm{C}\bm{\mu}\|_F^2$. Again, we use the projected gradient iteration $\bm{\mu}_{k+1} = [\bm{\mu}_k - \alpha \nabla_{\mu} F(\bm{\mu})]_+$ with gradient $\nabla_{\mu} F(\bm{\mu}) = \bm{C}^\top \bm{C} \bm{\mu} - \bm{C}^\top \bm{X}$, and Lipschitz constant $\|\bm{C}^\top \bm{C}\|_2$. The same stopping rule and maximum number of inner iterations are used for this update.

\subsection{Update of $L$} \label{subsec:update_L}
The set $\cL_{\cG}$ is non-convex, which precludes the use of standard convex optimization methods. Furthermore, the space of all possible graphs on $n$ nodes has cardinality $2^{\binom{n}{2}}$, making any explicit search over $\cL_{\cG}$ computationally infeasible for moderately large $n$.

To address this, we replace the deterministic optimization over $\cL_{\cG}$ by a stochastic approximation aimed at recovering the latent graph topology. Rather than optimizing over all possible Laplacians, we restrict attention to a parametric family of random graphs generated by SBM.
The Laplacian update then proceeds in two stages: a deterministic optimization of the model parameters, followed by a stochastic realization of the graph according to the SBM.

\paragraph{Stochastic Block Model (SBM) \cite{SBM}}
Stochastic Block Models are a classical family of probabilistic models for random graphs, in which the probability of an edge between two nodes depends on the community membership of those nodes. We adopt a similarity-modulated variant: for $n$ nodes, the probability of an edge between nodes $i$ and $j$ is:
\[
  P(A_{ij}=1) = B_{ij}(a,b) =
  \begin{cases}
    a\,S_{ij} & \text{if } \bm{x}_i,\bm{x}_j \text{ are in the same cluster}\\
    b\,S_{ij} & \text{if } \bm{x}_i,\bm{x}_j \text{ are in different clusters}  
  \end{cases}
\]
where $a>0$, $b \ge 0$, and $\bm{S} \in [0,1]^{n \times n}$ is a similarity matrix. The adjacency matrix $\bm{A}$ is sampled entry-wise as
\[
    A_{ij} \sim \operatorname{Ber}\bigl(B_{ij}(a,b)\bigr), \quad i \neq j,
\]
with no self-loops, where $\operatorname{Ber}(p)$ denotes the Bernoulli distribution with success probability $p$. The associated random Laplacian is $\bm{L} = \bm{D} - \bm{A}$, where $\bm{D} = \Diag(\bm{A}\ones_n)$. Under this construction, the Laplacian is a random variable parameterized by $\theta = (a,b)$.

\paragraph{Distributionally Robust Optimization (DRO) \cite{DRO}}
In classical stochastic optimization, decision-making under uncertainty is formulated as the problem of minimizing the expected cost under a fixed distribution: $\min_{x \in \cX} \IE_P[f(x,\xi)]$, where $x$ is the decision variables, $\xi$ is a random variables, and $P$ is the probability distribution governing the uncertainty. In practice, $P$ is rarely known and must be estimated from data; inaccurate estimation may lead to solutions that perform poorly when the true distribution deviates from the estimate. DRO addresses this by replacing the single assumed distribution with an \emph{ambiguity set} $\cB$ of plausible distributions:
\[
  \min_{x \in \cX} \max_{P \in \cB} \IE_P[f(x,\xi)].
\]
This seeks decisions that minimize the worst-case expected cost, yielding solutions robust to distributional uncertainty.

In our setting, the exact distribution governing the random graph is unknown. We adopt a DRO perspective and, restricting to the SBM parametric family, we define the ambiguity set as
\begin{equation}\label{set:probability_set_according_to_DRO}
  \cB = \Big\{
  \bm{B}(a,b) = a\bm{S} + b(\ones_n\ones_n^\top - \bm{S}) \ |\ 
        a > 0, b \geq 0, \|\bm{B}(a,b)-\bm{S}\|_F \leq r,   \bm{B}(a,b)\ones_n = \ones_n 
        \Big\}
\end{equation}
that is, the collection of SBM edge-probability matrices within distance $r$ of the empirical similarity matrix $\bm{S}$, and optimize against the worst-case expectation over $\cB$.

Let $\bm{T} = \bm{C}\bm{C}^\top$. The Laplacian update is formulated as
\[
  \min_{\theta=(a,b)}\max_{\bm{B} \in \cB}
  \IE_{\theta}\!\left[ \langle \bm{L}(\theta), \bm{T} \rangle \right].
\]
Since the objective is linear in $\bm{B}(a,b)$ and the ambiguity set $\cB$ imposes proximity constraints on $\bm{B}(a,b)$, the worst-case over $\cB$ translates into explicit bounds on the scalar parameter $b$, derived in Appendix~\ref{sec:appendix_constraints}. By the linearity of expectation, the objective then reduces to
\begin{equation}\label{eq:DRO}
  \min_{b \in [0,b_{\max}]}
  \langle \IE_{b}[\bm{L}], \bm{T} \rangle 
\end{equation}
where the expected Laplacian admits a closed-form expression. Under the SBM, each entry of $\bm{A}$ is Bernoulli, so
\begin{equation}\label{def:B}
\begin{array}{rcll}
\bm{B}(a,b)
=\bm{B}(b) \overset{\text{\footnotemark}}{=} \mathbb{E}_{\theta}[\bm{A}] &=& a \bm{S} + b (\ones_n\ones_n^\top - \bm{S}),
& \text{with } a = 1 + b(1-n),
\\[2mm]
      &=& (1 - nb) \bm{S} + nb \,\tfrac{1}{n}\ones_n\ones_n^\top
\\[2mm]
      &=& \mathrm{Conv}_b\!\left(\bm{S}, ~\tfrac{1}{n}\ones_n\ones_n^\top\right),
\end{array}
\end{equation}
\footnotetext{From the derivation of constraints in Appendix~\ref{sec:appendix_constraints}, row-stochasticity imposes $a = 1 + b(1-n)$. Hence, $a$ is fully determined by $b$, and all quantities that depend on $a$ can be expressed as functions of $b$ alone.}\
\noindent
where $\bm{S}$ encodes the community structure. This shows that $\bm{B}(b)$ lies on the line through $\bm{S}$ and $\tfrac{1}{n}\ones_n\ones_n^\top$, and in their convex hull, denoted $\mathrm{Conv}_b(\cdot,\cdot)$, whenever $b \leq 1/n$. Over the whole admissible range of $b$, the entries $B_{ij} = aS_{ij} + b(1-S_{ij})$ remain in $[0,1]$, being convex combinations of the scalars $a, b \in [0,1]$.

The expected degree matrix is $\IE_{\theta}[\bm{D}] = \diag(\bm{B} \ones_n)$. Combining these, the expected Laplacian is
\begin{equation}\label{eq:expected_laplacian}
\IE_{\theta}[\bm{L}]
~=~
\Diag\bigl(\bm{B}(b)\ones_n\bigr) - \bm{B}(b)
~\overset{\eqref{def:B}}{=}~
  \Diag\!\bigl(\bigl[(1 - nb) \bm{S} + b\ones_n\ones_n^\top]\ones_n\bigr)
  - (1 - nb) \bm{S} - b \ones_n\ones_n^\top.
\end{equation}
Assuming a row-stochastic similarity matrix, i.e., $\bm{S}\ones_n = \ones_n$,
one computes $\bm{B}(b)\ones_n = \ones_n$, so the expected degree matrix is constant: $\Diag(\bm{B}(b)\ones_n) = I$. The objective function in \eqref{eq:DRO} becomes
\begin{equation}\label{def:Fb}
  \Phi(b) = \langle I, \bm{T} \rangle - (1-nb)\langle \bm{S}, \bm{T} \rangle - b\langle \mathbf{1}_n\mathbf{1}_n^\top, \bm{T} \rangle = \alpha - \gamma b,
\end{equation}
where $\alpha := \langle I, \bm{T} \rangle - \langle \bm{S}, \bm{T} \rangle$ and $\gamma := \langle \mathbf{1}_n\mathbf{1}_n^\top, \bm{T} \rangle - n\langle \bm{S}, \bm{T} \rangle$ are constants independent of $b$. This is strictly linear in $b$. 

The admissible values of the parameter $b$ (derived in Appendix~\ref{sec:appendix_constraints}) satisfy
\begin{equation}\label{eq:b_bounds}
0 \le b \le b_{\max} \coloneqq \min\!\left\{
    \frac{1}{n-1}, \frac{r}{n\sqrt{S_2-1}}
  \right\},
\end{equation}
where $S_2 \coloneqq \sum_{i,j} S_{ij}^2$, thus the final objective is 
\[  
\min_{b \in [0, b_{\max}]} \Phi(b).
\]

\paragraph{Proximal Point Method \cite{proximal_point_method} on $\bm{b}$}

Since $\Phi(b)$ in \eqref{def:Fb} is linear, its unconstrained minimization always attains a boundary of the feasible interval, leading to a degenerate solution. To promote stable interior updates, we use a proximal point approach:
\begin{equation}\label{eq:prox_update}
  b_{k+1} = \argmin_{b \in [0,b_{\max}]} \Big\{
    \Phi(b) + \tfrac{1}{2\tau}(b - b_k)^2
  \Big\},
\end{equation}
where $\tau > 0$ is a stepsize. Setting the derivative of the regularized objective to zero gives the unconstrained update $b_{k+1}^{\mathrm{unproj}} = b_k + \tau\gamma$, and projecting onto the feasible interval yields
\[
b_{k+1} 
= 
\min\Bigl(\max(b_k + \tau\gamma, 0),
~ b_{\max}\Bigr).
\]
This update has a clear interpretation: when $\gamma > 0$, the objective decreases as $b$ increases, so the iterate is pushed toward $b_{\max}$; when $\gamma < 0$ it moves toward $0$. The proximal term ensures controlled updates and keeps $b_{k+1}$ within the admissible interval.

\section{Convergence Analysis}\label{sec:convergence}
The BCD scheme is interpreted as a discrete dynamical system on $(\bm{C},\bm{\mu},b)$. To analyze its stability and convergence, we introduce a \textbf{global Lyapunov functional} $\mathcal{J}(\bm{C},\bm{\mu},b)$, whose monotonic descent along the BCD iterations drives convergence \cite{Lyapunov}:
\begin{equation}\label{eq:lyapunov_functional}
\mathcal{J}(\bm{C},\bm{\mu},b)
=
\underbrace{\tr(\bm{C}^\top \bm{L}(b)\bm{C})}_{\text{Graph Dirichlet Energy}}
+
\underbrace{\lambda_{\bm{\mu}}\|\bm{X}-\bm{C}\bm{\mu}\|_F^2}_{\text{Data Fidelity}},
\end{equation}
which decomposes into two terms with distinct roles.

\begin{itemize}[leftmargin=*]\setlength{\itemsep}{4pt}

\item \textbf{Graph Dirichlet Energy.} In the continuous setting, the Dirichlet energy $E[u] = \int_\Omega \|\nabla u\|^2\,dx$ measures the total variation of $u:\Omega\to\mathbb{R}$: small values indicate a smooth, slowly varying function, while large values signal rapid oscillations. Its discrete analogue on a weighted graph with Laplacian $\bm{L} = \bm{D}-\bm{A}$ is $\tr(\bm{C}^\top \bm{L} \bm{C})$. In our stochastic setting, the graph topology is uncertain: replacing $\bm{L}$ with the expected Laplacian $\bm{L}(b) = \mathbb{E}_b[\bm{L}(B)]$ defined in~\eqref{eq:expected_laplacian} yields a smoothness term robust to edge-level noise, encouraging cluster assignments stable across likely graph realizations.

\item \textbf{Data Fidelity.} This term ensures that the cluster centers $\bm{C}$ remain close to the observed data $\bm{X}$, with $\lambda_\mu$ controlling the balance between smoothness and data fit.

\end{itemize}

\paragraph{Lower Boundedness} The functional $\cJ$ is bounded from below.
\begin{itemize}[leftmargin=*]\setlength{\itemsep}{0pt}
  \item The data-fidelity term is non-negative by definition.
    \item If $b \ge 0$, then $\bm{B}(b)$ has non-negative entries, and it is symmetric since $\bm{S}$ is. The Laplacian quadratic form satisfies
    \[ 
    \tr(\bm{C}^\top \bm{L}(b)\bm{C}) = \tfrac12 \sum_{k} \sum_{i,j} B_{ij}(b)(C_{ik}-C_{jk})^2\ge 0.
    \]
\end{itemize}
Thus $\cJ(\bm{C},\bm{\mu},b) \geq 0$, and the sequence generated by the algorithm cannot diverge to $-\infty$.

\paragraph{Monotonic Descent under BCD}
Since each block update is performed via projected gradient descent with a step size equal to the inverse Lipschitz constant (\cref{subsec:update_C_mu}), the standard descent lemma for smooth functions gives:
\[
\begin{array}{rll}
\text{Updating } \bm{C} \text{ with } (\bm{\mu},b) \text{ fixed:}
  &&\cJ(\bm{C}_{k+1},\bm{\mu}_k,b_k) + \tfrac{L_C}{2}\|\bm{C}_{k+1}-\bm{C}_k\|_F^2 \leq \cJ(\bm{C}_k,\bm{\mu}_k,b_k),
  \\[4pt]
\text{Updating } \bm{\mu} \text{ with } (\bm{C},b) \text{ fixed:}
  &&\cJ(\bm{C}_{k+1},\bm{\mu}_{k+1},b_k) + \tfrac{L_\mu}{2}\|\bm{\mu}_{k+1}-\bm{\mu}_k\|_F^2
    \leq \cJ(\bm{C}_{k+1},\bm{\mu}_k,b_k).
\end{array}
\]

\paragraph{Expected Descent in $\mathbf{b}$}
Updating $b$ while keeping $(\bm{C},\bm{\mu})$ fixed requires special care due to the stochastic nature of the graph update; the monotonic decrease $\cJ(\bm{C}_{k+1},\bm{\mu}_{k+1},b_{k+1}) \leq \cJ(\bm{C}_{k+1},\bm{\mu}_{k+1},b_k)$ does not hold pathwise, we instead analyze the expected energy.\footnote{The update does not automatically guarantee the monotonic decrease
$\cJ(\bm{C}_{k+1},\bm{\mu}_{k+1},b_{k+1})
\le
\cJ(\bm{C}_{k+1},\bm{\mu}_{k+1},b_k),
$ since the Laplacian depends on a random realization of the graph.
The expected variant is $\IE\!\left[ \cJ(\bm{C}_{k+1},\bm{\mu}_{k+1},b_{k+1})
-
\cJ(\bm{C}_{k+1},\bm{\mu}_{k+1},b_k)
\right]$.
} The key observation is that, in expectation, $\cJ$ depends on $b$ only through the function $\Phi$ in \eqref{def:Fb}, which is exactly the function minimized by the proximal update \eqref{eq:prox_update}.

\begin{theorem}[Expected descent w.r.t.\ the graph parameter]\label{thm:stochastic_convergence}
Fix a row-stochastic $\bm{S} \in \IRnn$, $\bm{C}_{k+1}$ and $\bm{\mu}_{k+1}$, define the Lyapunov energy $\cJ$ as in \eqref{eq:lyapunov_functional}, and assume that the graph Laplacian $\bm{L}(b)$ is generated according to an SBM. Let $b_{k+1}$ be given by the proximal update \eqref{eq:prox_update} with $\bm{T} = \bm{C}_{k+1}\bm{C}_{k+1}^\top$. Then
\[
  \IE\bigl[\cJ(\bm{C}_{k+1},\bm{\mu}_{k+1},b_{k+1})\bigr]
  {}+ \tfrac{1}{2\tau}(\Delta b_k)^2
  \leq
  \IE\bigl[\cJ(\bm{C}_{k+1},\bm{\mu}_{k+1},b_k)\bigr],
\]
where $\Delta b_k \coloneqq b_{k+1}-b_k$.
\end{theorem}

\begin{proof}
The only random term in $\cJ$ is the Laplacian contribution. By linearity of trace and expectation:
\[
  \IE\bigl[\tr(\bm{C}^\top \bm{L}(b) \bm{C})\bigr] = \tr\bigl(\bm{C}^\top \IE[\bm{L}(b)] \bm{C}\bigr)
  {}= \langle \IE[\bm{L}(b)], \bm{C}\bm{C}^\top \rangle.
\]
For $\bm{C} = \bm{C}_{k+1}$, i.e., $\bm{T} = \bm{C}_{k+1}\bm{C}_{k+1}^\top$, this is exactly $\Phi(b)$ in \eqref{def:Fb}. The data-fidelity term does not depend on $b$, hence
\[
  \IE\bigl[\cJ(\bm{C}_{k+1},\bm{\mu}_{k+1},b)\bigr] = \Phi(b) + \mathrm{const}.
\]
Since $b_{k+1}$ minimizes $\Phi(b) + \tfrac{1}{2\tau}(b-b_k)^2$ over $[0,b_{\max}]$ and $b_k \in [0,b_{\max}]$ is feasible, comparing the objective at $b_{k+1}$ and at $b_k$ (where the proximal term vanishes) gives
\[
  \Phi(b_{k+1}) + \tfrac{1}{2\tau}(\Delta b_k)^2 \leq \Phi(b_k),
\]
which is the claim.
\end{proof}
The theorem shows that, in expectation, the stochastic graph update does not increase the Lyapunov functional, and that the decrease is at least quadratic in the step size.

\paragraph{Total Expected Descent}
Combining the deterministic descent inequalities for $\bm{C}$ and $\bm{\mu}$ with Theorem~\ref{thm:stochastic_convergence}, we obtain the final combined inequality\footnote{By Theorem~\ref{thm:stochastic_convergence},
$\IE\big[\cJ(\bm{C}_{k+1},\bm{\mu}_{k+1},b_{k+1})\big]
+ \tfrac{1}{2\tau}(\Delta b_k)^2
\leq
\IE\big[\cJ(\bm{C}_{k+1},\bm{\mu}_{k+1},b_k)\big]$ for all $k$.
Then the standard gradient descent lemma for smooth functions with Lipschitz constants $L_C$ and $L_\mu$ gives
\[
\cJ(\bm{C}_{k+1}, \bm{\mu}_{k+1}, b_k) 
+ \frac{L_\mu}{2} \|\bm{\mu}_{k+1}-\bm{\mu}_k\|_F^2
+ \frac{L_C}{2} \|\bm{C}_{k+1}-\bm{C}_k\|_F^2
\le \cJ(\bm{C}_k, \bm{\mu}_k, b_k).
\]
So, adding the two inequalities, we get the final combined inequality \eqref{eq:combined_descent}.
}:
\begin{equation}\label{eq:combined_descent}
  \IE\bigl[\cJ_{k+1}\bigr]
  + \frac{1}{2\tau}(\Delta b_k)^2
  + \frac{L_\mu}{2}\|\bm{\mu}_{k+1}-\bm{\mu}_k\|_F^2
  + \frac{L_C}{2}\|\bm{C}_{k+1}-\bm{C}_k\|_F^2
  ~\leq~
  \IE\bigl[\cJ_k\bigr],
\end{equation}
where $\cJ_k \coloneqq \cJ(\bm{C}_k,\bm{\mu}_k,b_k)$. Hence $(\IE[\cJ_k])_k$ is nonincreasing and, being bounded below by $0$, convergent. Moreover, summing \eqref{eq:combined_descent} over $k$ gives
\[
  \sum_{k\geq 0}
  \left(
  \frac{1}{2\tau}(\Delta b_k)^2 
  + 
  \frac{L_\mu}{2}\|\bm{\mu}_{k+1}-\bm{\mu}_k\|_F^2 
  + 
  \frac{L_C}{2}\|\bm{C}_{k+1}-\bm{C}_k\|_F^2
  \right) 
  \leq \IE[\cJ_0] < \infty,
\]
so that the successive differences of all blocks vanish as $k \to \infty$.

\section{Persistence-based Similarity Matrix}
\label{sec:persistence_based_similarity_matrix}
The framework developed above operates on a generic similarity matrix $\bm{S}$, which simultaneously parametrizes the SBM edge-probability matrix and defines the center of the DRO ambiguity set. The quality of this prior is therefore consequential: it determines both the graph distribution from which the latent graph is drawn and the region of distributional uncertainty over which the optimization is robust. We propose to instantiate $\bm{S}$ from the zero-dimensional persistent homology of the data, grounding both roles in the multiscale connectivity structure of the data. 

Let $\{K_\varepsilon\}_{\varepsilon \geq 0}$ be a finite increasing filtration of simplicial complexes on $\{\bm{x}_1, \ldots, \bm{x}_n\}$. The inclusions $K_\varepsilon \hookrightarrow K_{\varepsilon'}$, for $\varepsilon \leq \varepsilon'$, induce linear maps $(i_{\varepsilon,\varepsilon'})_{0*}: H_0(K_\varepsilon) \to H_0(K_{\varepsilon'})$ between the zeroth simplicial homology groups, yielding the zero-dimensional persistent homology as the persistence module $\{H_0(K_\varepsilon), (i_{\varepsilon,\varepsilon'})_{0*}\}_{\varepsilon \leq \varepsilon'}$. The complete record of the birth and death of connected components across the filtration is encoded in the $H_0$ barcode. For each pair $i,j$, define
\[
\varepsilon_{ij} := \inf\{\varepsilon \geq 0 : [\bm{x}_i]_\varepsilon = 
    [\bm{x}_j]_\varepsilon \text{ in } H_0(K_\varepsilon)\}
\]
as the death time of the $H_0$ barcode interval corresponding to the merger of the connected components containing $\bm{x}_i$ and $\bm{x}_j$. By the Structure Theorem 
\cite{zomorodian2005}, the $H_0$ barcode is the unique decomposition of the 
persistence module into interval modules, and $\varepsilon_{ij}$ is read off 
unambiguously as the right endpoint of the interval corresponding to the merger of $\bm{x}_i$ and $\bm{x}_j$; hence $\varepsilon_{ij}$ is a well-defined invariant of the 
persistence module. The persistence-based similarity matrix $\bm{S} \in [0,1]^{n \times n}$ is then defined as
\[
    S_{ij} := 1 - \frac{\varepsilon_{ij}}{\varepsilon_{\max}}, \quad 
    \varepsilon_{\max} := \max_{i,j} \varepsilon_{ij}\footnote{The filtration is computed up to a finite threshold, so that every pair of points is eventually merged into the same connected component; this guarantees $\varepsilon_{ij} < \infty$ for all $i,j$, and hence that $\varepsilon_{\max}$ is a well-defined finite quantity.}.
\]

\begin{observation}
$\bm{S}$ is a well-defined similarity matrix: $S_{ii} = 1$ since $\varepsilon_{ii} = 0$, as each point belongs to its own connected component at $\varepsilon = 0$; symmetry follows from $\varepsilon_{ij} = \varepsilon_{ji}$, since component merger is a symmetric event; and $S_{ij} \in [0,1]$ by definition of $\varepsilon_{\max}$.
\end{observation}

This construction is the natural choice for the role $\bm{S}$ plays in the framework for two reasons. First, the similarity encoded in $\bm{S}$ is topological: $S_{ij}$ is large precisely when $\bm{x}_i$ and $\bm{x}_j$ merge into the same connected component at an early filtration scale, reflecting the global multiscale connectivity structure of the data rather than proximity at any fixed scale. This aligns directly with the graph-cut regularization, which promotes cluster assignments that respect this connectivity. 
Second, the Stability Theorem for persistent homology \cite{cohen-steiner2007} guarantees that the $H_0$ barcode varies continuously with the input data in the bottleneck distance, which implies that $\varepsilon_{ij}$, and hence $\bm{S}$, inherit 
this stability: small perturbations of the data produce small perturbations in $\bm{S}$.
By feeding~$\bm{S}$ into both the SBM parametrization and the DRO ambiguity set, the graph-learning process is anchored to the multiscale connectivity structure of the data, and this structure is propagated, through the graph-cut term, directly into the clustering objective.

\section{Numerical Experiments} \label{sec:implementation}
To validate T-ARC, we conducted experiments on both synthetic and real datasets. Five synthetic datasets were designed to evaluate the method under controlled conditions, each highlighting specific structural properties. Additionally, we assessed performance on random subsets of Fashion-MNIST \cite{fashion_MNIST_dataset}, providing a more realistic benchmark in a high-dimensional setting.

\begin{algorithm}[h]
\caption{T-ARC}
\label{alg:proposed_model}
\KwInput{Data matrix $\bm{X} \in \IR^{n\times d}$, number of clusters $k$, parameter $\lambda_\mu$, ambiguity radius $r$}
Initialize $\bm{C} \geq 0$, $\bm{\mu} \geq 0$.
Initialize $b$ in \eqref{eq:b_bounds}.
Construct persistence-based similarity matrix $\bm{S}$ as in Section \ref{sec:persistence_based_similarity_matrix}.\\
Apply symmetric Sinkhorn--Knopp normalization to $\bm{S}$, so that it is simultaneously row-stochastic and symmetric.\\

\For{$t = 0,1,2,\dots$}{
  \textbf{Batch update $\mathbf{\bm{C}}$:}
  Run projected gradient steps on $\bm{C}$ minimizing $F(\bm{C},\bm{\mu}^{(t)},\bm{L}^{(t)})$, until a stopping criterion is met.\\
  \textbf{Batch update $\mathbf{\bm{\mu}}$:}
  Run projected gradient steps on $\bm{\mu}$ minimizing $F(\bm{C}^{(t+1)},\bm{\mu},\bm{L}^{(t)})$, until a
  stopping criterion is met.\\
  \textbf{Proximal update on $\mathbf{b}$:}
  Compute $b^{(t+1)}$ as \eqref{eq:prox_update}, set $a^{(t+1)} = 1 + b^{(t+1)}(1-n)$.\\
  \textbf{Monte Carlo Laplacian update:}
  Construct cluster-aware probabilities $\bm{B}$ from $\bm{C}^{(t+1)}$; 
  sample $m_{\mathrm{MC}}$
  adjacency matrices $\bm{A}^{(m)} \sim \mathrm{Ber}(\bm{B})$, 
  compute mean $\overline{\bm{A}}$, threshold to binary, 
  set $\bm{L}^{(t+1)} = \diag(\overline{\bm{A}}\ones) - \overline{\bm{A}}$.\\
}
\Return{Cluster assignment matrix $\bm{C}^{(t+1)}$, centroids $\bm{\mu}^{(t+1)}$}
\end{algorithm}

The core optimization algorithm is described in Algorithm~\ref{alg:proposed_model}. Specifically it firstly applies symmetric Sinkhorn--Knopp normalization \cite{sinkhorn1967concerning} to the input similarity matrix $\bm{S}$. This normalization iteratively replaces $\bm{S}$ with $\bm{D}^{-1/2}\bm{S}\bm{D}^{-1/2}$, where $\bm{D}=\diag(\bm{S}\ones_n)$, until convergence, guaranteeing that $\bm{S}$ is simultaneously row-stochastic and symmetric before entering the BCD loop, ensuring the theoretical assumption. This normalization step is independent of the choice of $\bm{S}$, allowing the framework to be instantiated with any chosen similarity matrix. 

As baseline methods, we consider classical clustering approaches spanning different modeling assumptions. In particular, we include the standard K-means algorithm, which provides a natural reference point since T-ARC can be interpreted as a topology-aware extension of K-means. We further consider spherical K-means, which is better suited for data where angular similarity is more informative than Euclidean distance, and spectral clustering, which leverages graph-based representations and is therefore closely related to our similarity-driven framework. Moreover, comparing T-ARC with its variant employing a similarity matrix based on the Euclidean norm allows for a direct assessment of the impact of the persistence-based construction. This comparison isolates the contribution of topological information, highlighting whether incorporating persistence-based connectivity leads to improved clustering performance.
The experimental design highlights two complementary aspects. First, on simple synthetic datasets, we verify that the proposed method behaves consistently with classical approaches when the cluster structure is well separated. Second, on more challenging settings --- synthetic datasets characterized by nonlinear geometries or latent topological structures --- we demonstrate that the proposed framework captures relationships that cannot be properly modeled by purely geometric clustering methods such as K-means. In addition, on random subsets of Fashion-MNIST, we assess its behavior on real high-dimensional data.
The proposed algorithm was implemented in \texttt{MATLAB}. Synthetic dataset generation and computation of the persistence-based similarity matrices were carried out in \texttt{Python} using the \texttt{GUDHI} library \cite{gudhi}. Baseline comparisons include classical K-means, Spherical K-means, and Spectral clustering and are performed in \texttt{MATLAB}.

\paragraph{Initialization of variables} 
The assignment matrix $\bm{C}_0$ is generated as a random matrix of size $n \times k$, normalized row-wise to ensure each row sums to one (with a regularization term of $10^{-14}$ to avoid division by zero). Initial centroids are $\bm{\mu}_0 = (\bm{C}_0^\top \bm{C}_0)^{-1} \bm{C}_0^\top \bm{X}$, rectified to enforce non-negativity. The parameter $b_0$ is sampled uniformly in $[0, b_{\max}]$ (from \eqref{eq:b_bounds}), and $a_0 = \max(10^{-3},\, 1 + b_0(1-n))$. The initial Laplacian $\bm{L}_0$ is constructed via Monte Carlo from the initialized clusters.

\paragraph{General Settings}
The behavior of T-ARC is controlled by several key parameters. The regularization parameter is $\lambda_\mu = 0.01$, emphasizing the topological structure over the K-means component. The Frobenius radius was chosen as $r = 0.01$, which defines the feasible region for the optimization of the probability variable. Monte Carlo sampling uses $m_{\mathrm{MC}} = 30$ samples. For the stopping condition, the global stopping threshold is $\epsilon = 10^{-6}$; inner batches use $\epsilon = 10^{-3}$. The proximal step-size is $\tau = 10^{-2}$. The random seed is fixed to $55$ for reproducibility. For the persistence-based similarity matrix, the Vietoris-Rips filtration \cite{TDA} is used. The resulting similarity matrices for all five experiments are shown in Fig.~\ref{fig:similarity_examples}.
\begin{figure}[htb]
    \centering
    \includegraphics[width=1\linewidth]{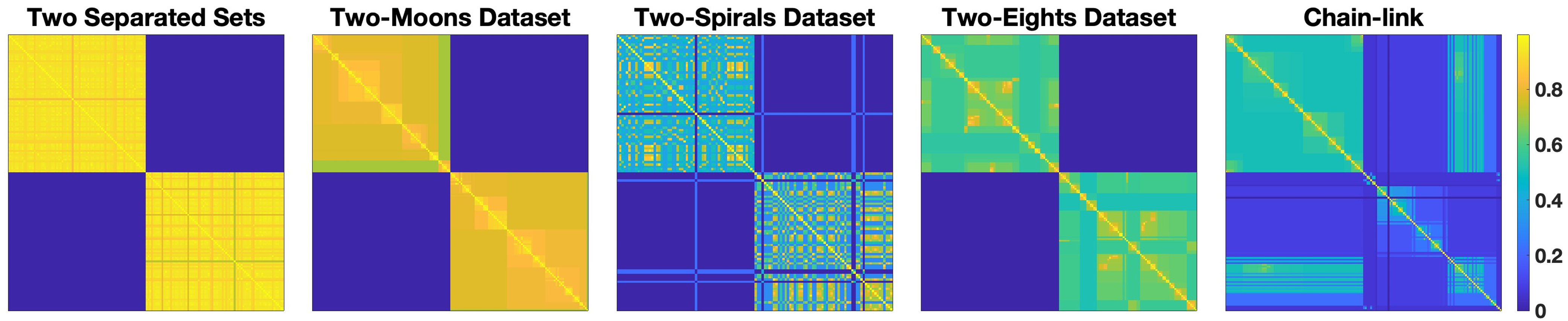}
    \caption{Persistence-based similarity matrices for the five experimental settings.}
    \label{fig:similarity_examples}
\end{figure}

\paragraph{Evaluation Metrics}
We report external and internal clustering validity indices. As external metrics we consider \textbf{Accuracy} and \textbf{F1-score} measure label agreement against ground-truth. For internal metric, we use the \textbf{Silhouette coefficient} \cite{rousseeuw1987}, defined as the average over all points of
\[
  s(i) = \frac{b(i) - a(i)}{\max\{a(i),b(i)\}},
\]
where $a(i)$ is the mean intra-cluster Euclidean distance and $b(i)$ is the minimum mean Euclidean distance to any other cluster.
Values lie in $[-1,1]$, with higher values indicating better clustering. The combination of internal and external metrics provides a comprehensive assessment of clustering performance, capturing both label agreement and intrinsic cluster structure.

\paragraph{Example 1: Two Sets Dataset \label{ex:ex1}}
We generated a two-dimensional point cloud of $160$ points divided equally into two clusters, centered at $(0.3,0.3)$ and $(0.7,0.7)$ within the unit square. The distance between the centers is fixed ($d\approx0.57$), while the isotropic Gaussian noise is set to $\sigma=0.05$, $0.12$ and $0.20$, so that the two sets are well separated, closer and overlapping, respectively, as shown in Fig.~\ref{fig:ex1_noise}.

\begin{figure}[h!]
    \centering
    \includegraphics[width=\linewidth]{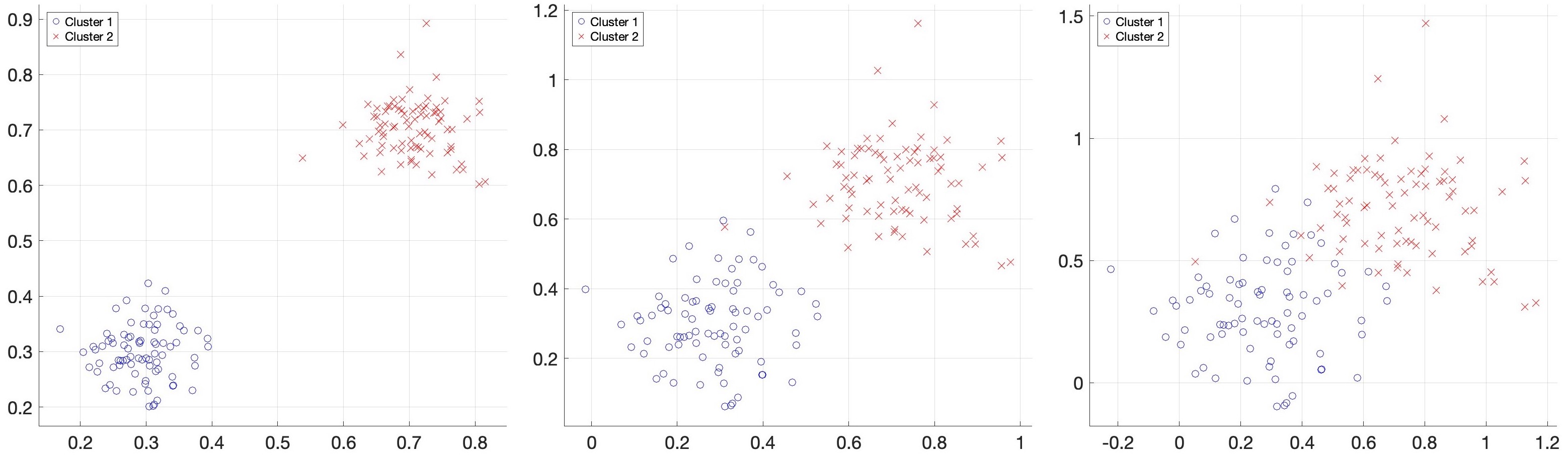}
    \caption{Two sets dataset obtained increasing isotropic Gaussian noise $\sigma$, creating well separated sets (left), closer sets (center) and overlapping sets (right)}
    \label{fig:ex1_noise}
\end{figure}

\begin{figure}[h!]
        \centering
        \includegraphics[width=\linewidth]{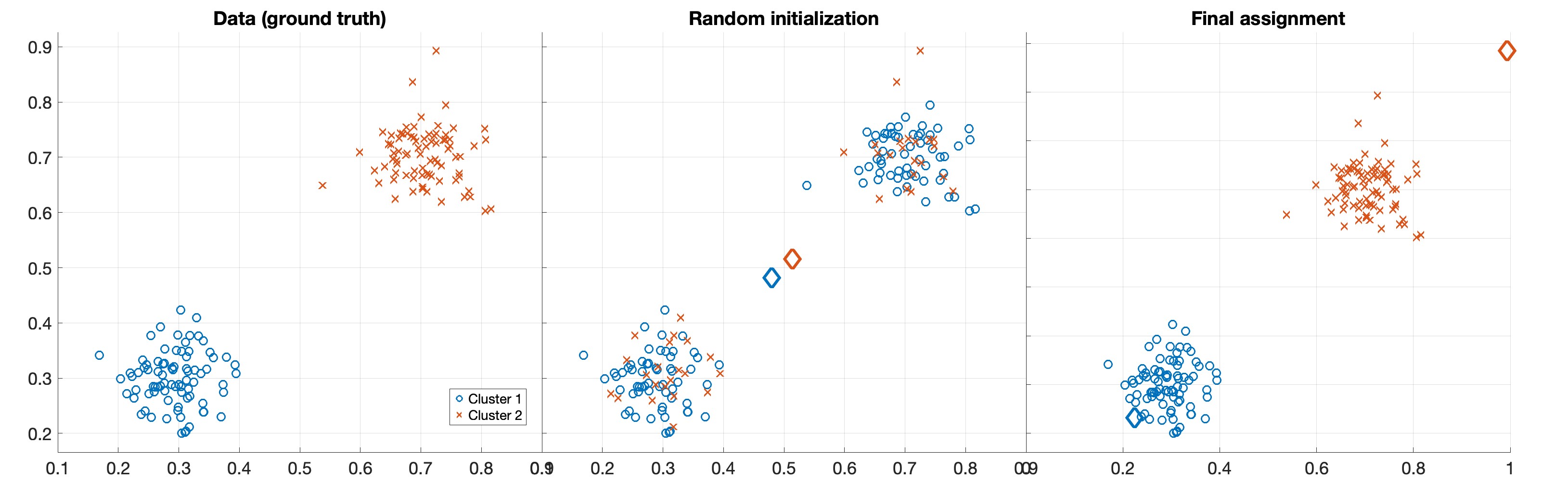}
  \caption{Three-panel illustration of the two separated sets dataset obtained using $\sigma=0.05$.
Left: the original data points, colored by ground-truth labels.
Center: random initialization with $k=2$ clusters.
Right: final assignment obtained by T-ARC using the persistence-based similarity matrix.}
  \label{fig:example1}
\end{figure}

\begin{figure}[h!]
  \centering
  \begin{subfigure}{0.49\linewidth}
    \centering
    \includegraphics[width=\linewidth]{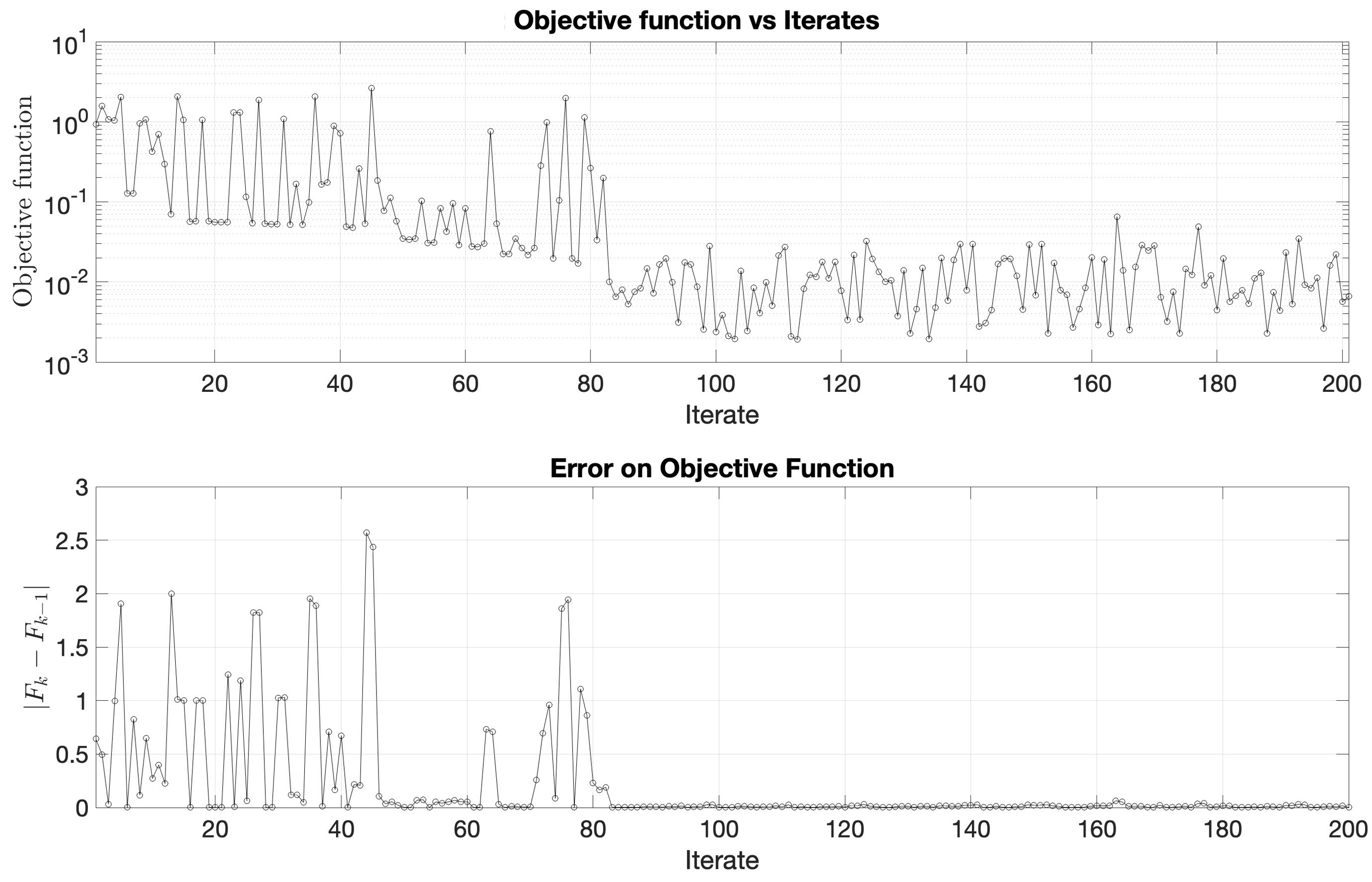}
    \caption{}
    \label{fig:ex1_objective}
  \end{subfigure}
  \hfill
  \begin{subfigure}{0.49\linewidth}
    \centering
    \includegraphics[width=\linewidth]{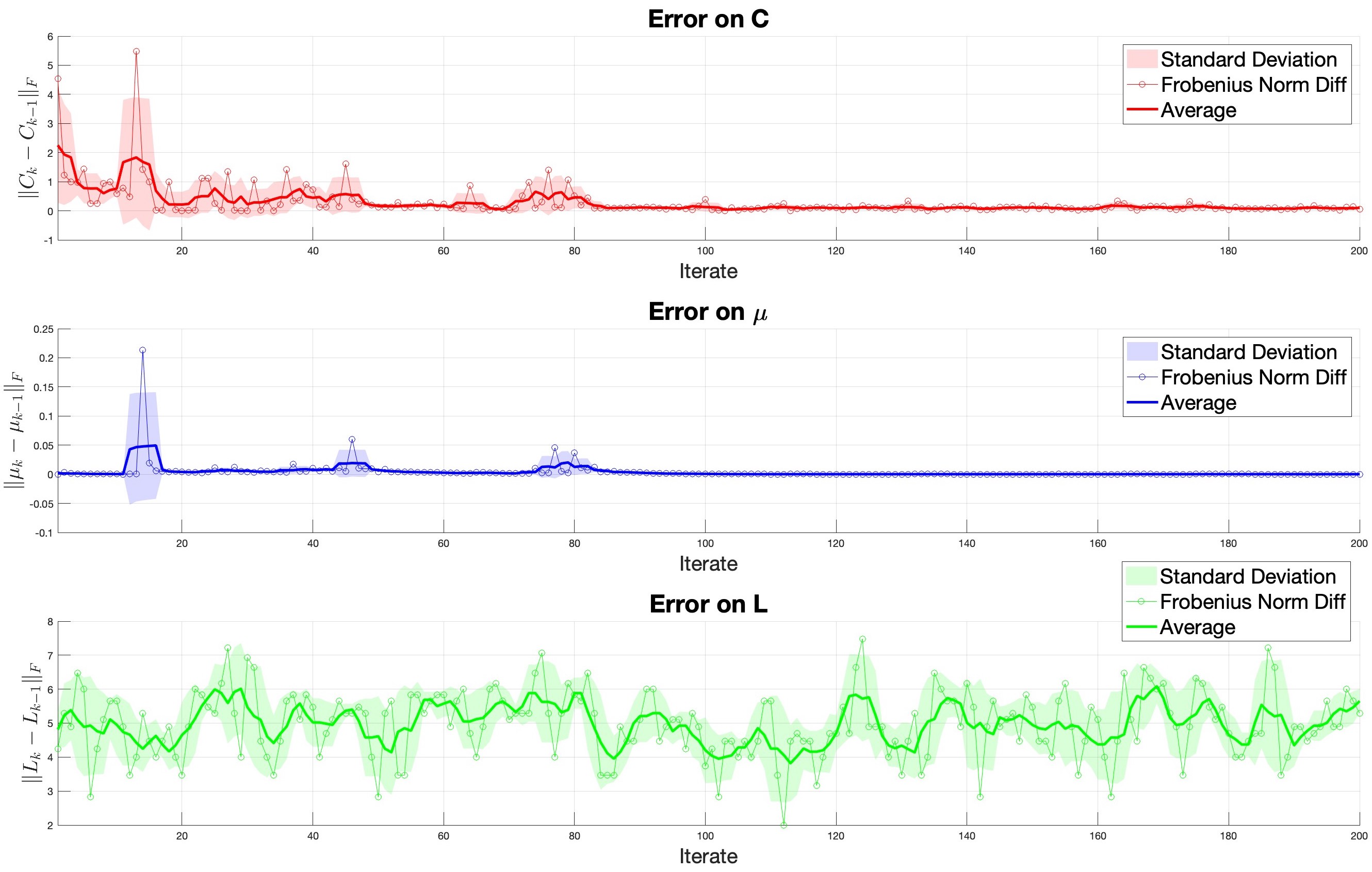}
    \caption{}
    \label{fig:ex1_error}
  \end{subfigure}
  \caption{(a) Objective function value during iterations and (b) Frobenius norm error of $\bm{C}$, $\bm{\mu}$, and $\bm{L}$ during iterations for T-ARC with persistence-based similarity matrix on Example 1.}
  \label{fig:ex1}
\end{figure}

When the sets are well separated ($\sigma=0.05$), T-ARC with persistence-based similarity achieves perfect Accuracy and F1-score ($1.00$), matching K-means, Spectral clustering and its Euclidean-based variant (Table~\ref{tab:results}). The Silhouette scores confirm high cluster coherence for all methods ($0.97$) except Spherical K-means ($0.02$), whose angular-similarity assumption does not match this data structure; Spherical K-means remains near $0.51$ Accuracy for all values of $\sigma$. We use this configuration as the reference one, illustrated in Fig.~\ref{fig:example1} and Fig.~\ref{fig:ex1}, which show the dataset, initialization, and final assignment, and the objective/error trajectories, respectively. Fig.~\ref{fig:ex1_objective} shows the objective decreasing and stabilizing, in agreement with the theory, while Fig.~\ref{fig:ex1_error} shows the errors on the variables.

As the sets get closer ($\sigma=0.12$), T-ARC (persistence) still matches K-means and Spectral clustering ($0.99$), while the Euclidean-based variant is slightly lower ($0.96$). When the sets overlap ($\sigma=0.20$), points drawn from the tail of one cluster fall in the region of the other, and the Accuracy of T-ARC (persistence) drops to $0.83$, below K-means ($0.93$), Spectral clustering ($0.95$) and the Euclidean-based variant ($0.92$). This behavior reflects the very property that makes the persistence-based prior effective on non-convex structures: since the similarity between two points is determined by the minimax linking path between them in the Vietoris--Rips filtration, it follows the connectivity of the data rather than its average geometry, which allows T-ARC to trace elongated and interleaved clusters (as shown in the following). When the two sets overlap, however, a few points in the overlap region are enough to connect them, and the connectivity signal no longer separates the clusters; in this regime the Euclidean-based variant, which aggregates all pairwise distances, is more robust. The two similarity constructions thus play complementary roles, and this experiment delineates the regime in which the persistence-based prior is most informative, namely when the persistence gap between intra- and inter-cluster merges is preserved.

\paragraph{Example 2: Two-Moons Dataset}\label{ex:ex2} 
We generated a two-dimensional point cloud of $160$ points forming two non-linearly separable half-moon clusters, centered around $(0.3,0.6)$ and $(0.7,0.3)$, with Gaussian noise $\sigma = 0.02$. Fig.~\ref{fig:ex2} shows the dataset, initialization, and final assignment.
\begin{figure}[h!]
\includegraphics[width=\linewidth]{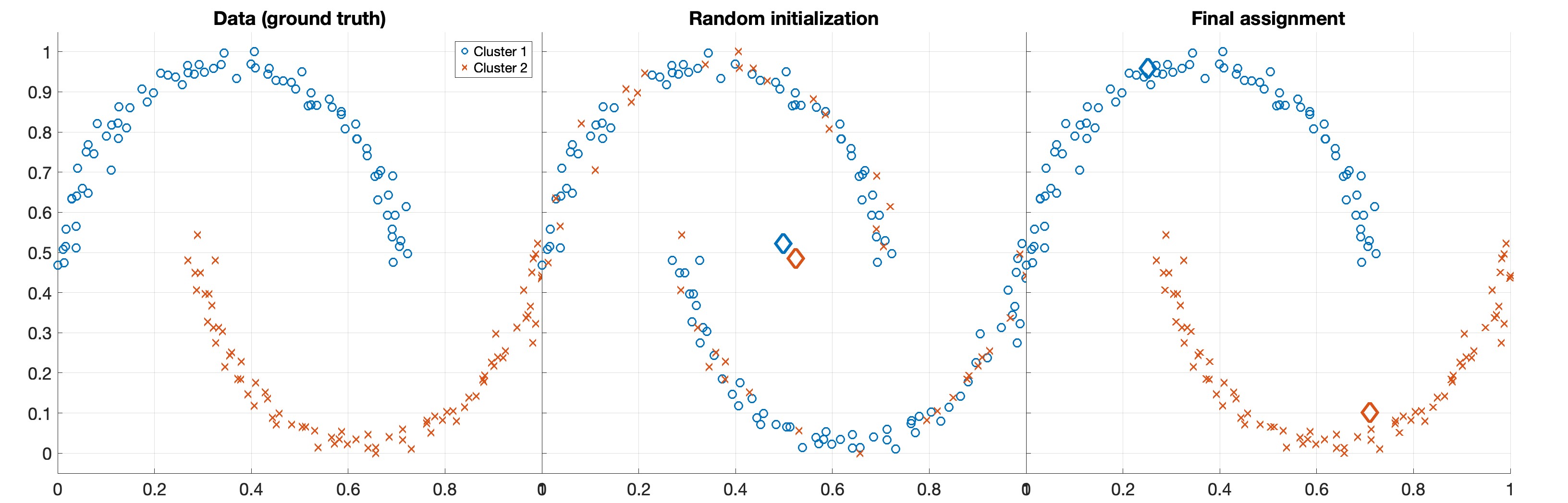}
  \caption{Three-panel illustration of the Two-Moons dataset.
Left: the original data points, colored by ground-truth labels.
Center: random initialization with $k=4$ clusters.
Right: final assignment obtained by T-ARC using the persistence-based similarity matrix.}
  \label{fig:ex2}
\end{figure}

\begin{figure}[h!]
  \centering
  \begin{subfigure}{0.49\linewidth}
    \centering
    \includegraphics[width=\linewidth]{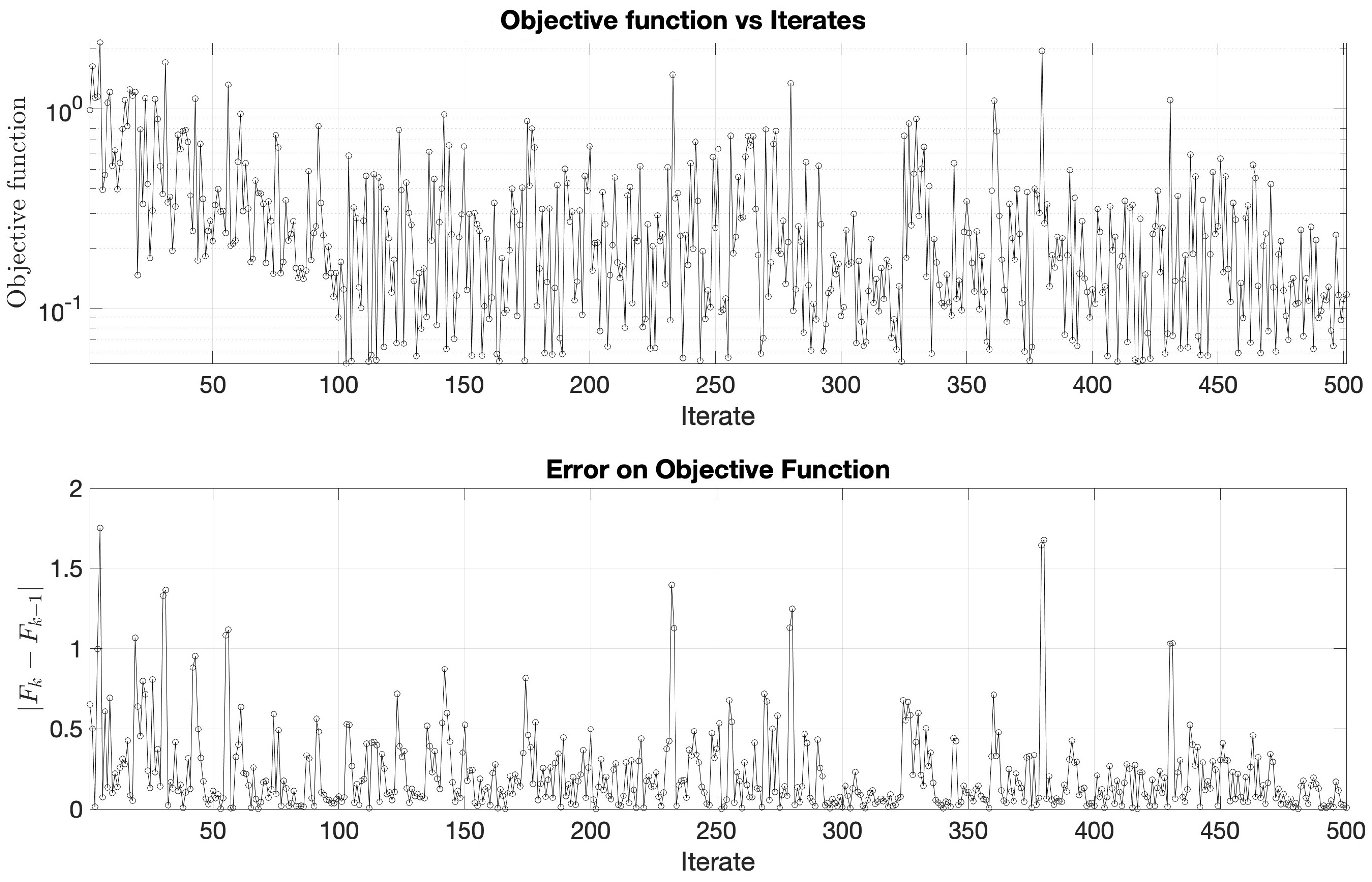}
    \caption{}
    \label{fig:ex2_objective}
  \end{subfigure}
  \hfill
  \begin{subfigure}{0.49\linewidth}
    \centering
    \includegraphics[width=\linewidth]{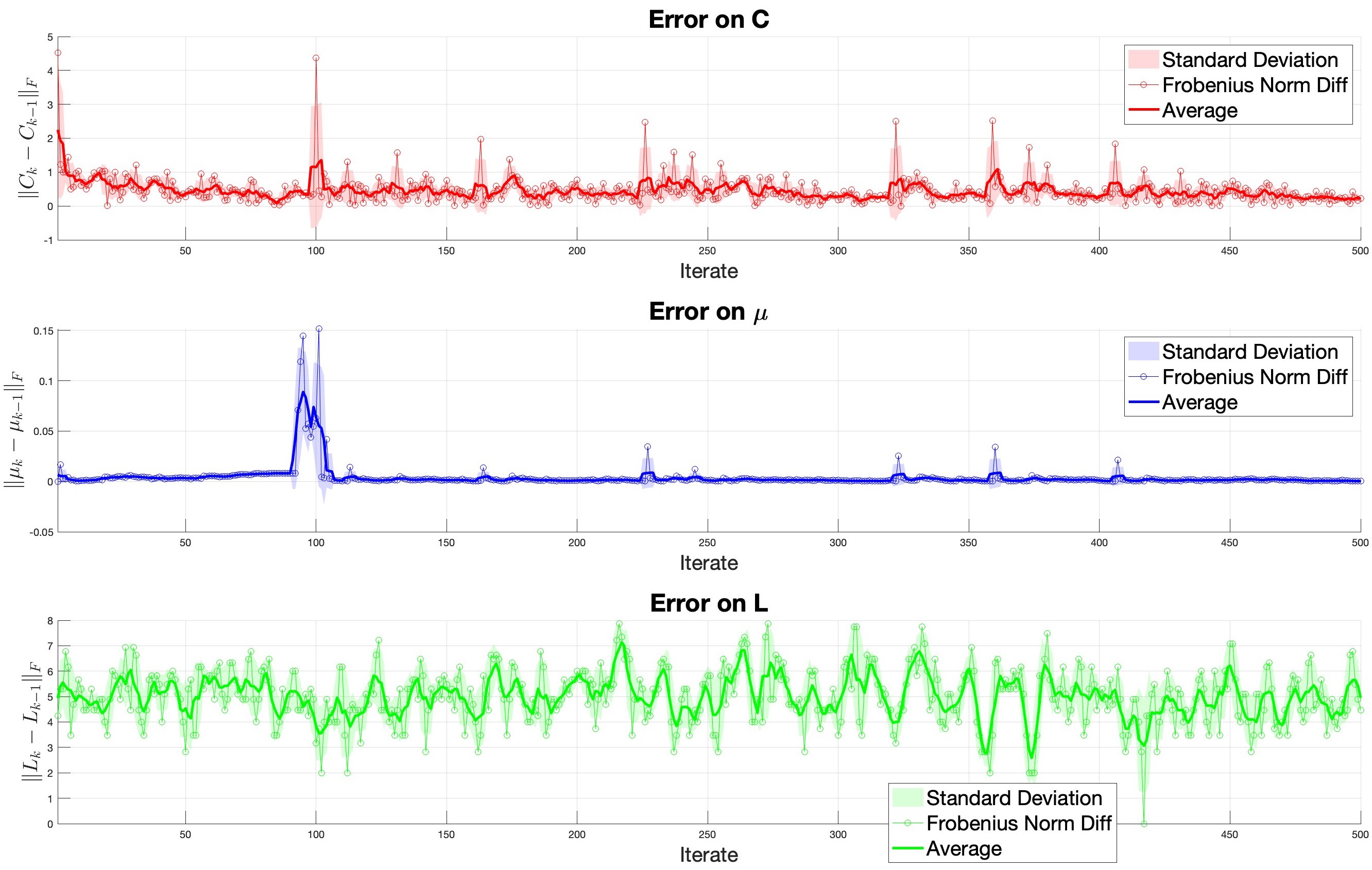}
    \caption{}
    \label{fig:ex2_error}
  \end{subfigure}
  \caption{(a) Objective function value during iterations and (b) Frobenius norm error of $\bm{C}$, $\bm{\mu}$, and $\bm{L}$ during iterations for T-ARC with persistence-based similarity matrix on Example 2.}
  \label{fig:example2}
\end{figure}
This dataset is known to be challenging for linear clustering models such as K-means. Incorporating the randomized Laplacian guided by the persistence-based similarity matrix enables T-ARC to follow the topological structure of the data, accurately identifying both connected components. Both T-ARC (persistence) and Spectral clustering achieve perfect Accuracy and F1-score ($1.00$), while K-means ($0.89$) and its Euclidean-based T-ARC variant ($0.87$) fall short. The Silhouette score of T-ARC (persistence) ($0.61$) is slightly below K-means ($0.67$), which is expected since the Silhouette coefficient measures cluster compactness in Euclidean space, favoring the convex partitions produced by K-means. Despite this, T-ARC correctly identifies the curved structure of the data, whereas K-means produces geometrically compact but topologically incorrect clusters. The Euclidean-based T-ARC version ($0.66$) achieves a higher Silhouette than the persistence-based variant, reflecting its more geometric nature, while Spectral clustering ($0.61$) matches T-ARC (persistence), confirming that the persistence prior guides the algorithm toward a topologically coherent partition rather than a purely geometric one.

\paragraph{Example 3: Two-Spirals Dataset}\label{ex:ex_spiral}
We generated a two-dimensional point cloud of $120$ points ($60$ per cluster) forming two interleaved Archimedean spirals, the second obtained by rotating the first by $\pi$. Each arm is parametrized as $(t\cos t,\, t\sin t)$ with $t \in [\pi/2,\, 2\pi]$, sampled with approximately uniform density along the arc, and isotropic Gaussian noise with $\sigma = 0.03$ is added to each point. Fig.~\ref{fig:ex_spiral} shows the dataset, initialization, and final assignment.
\begin{figure}[h!]
\includegraphics[width=\linewidth]{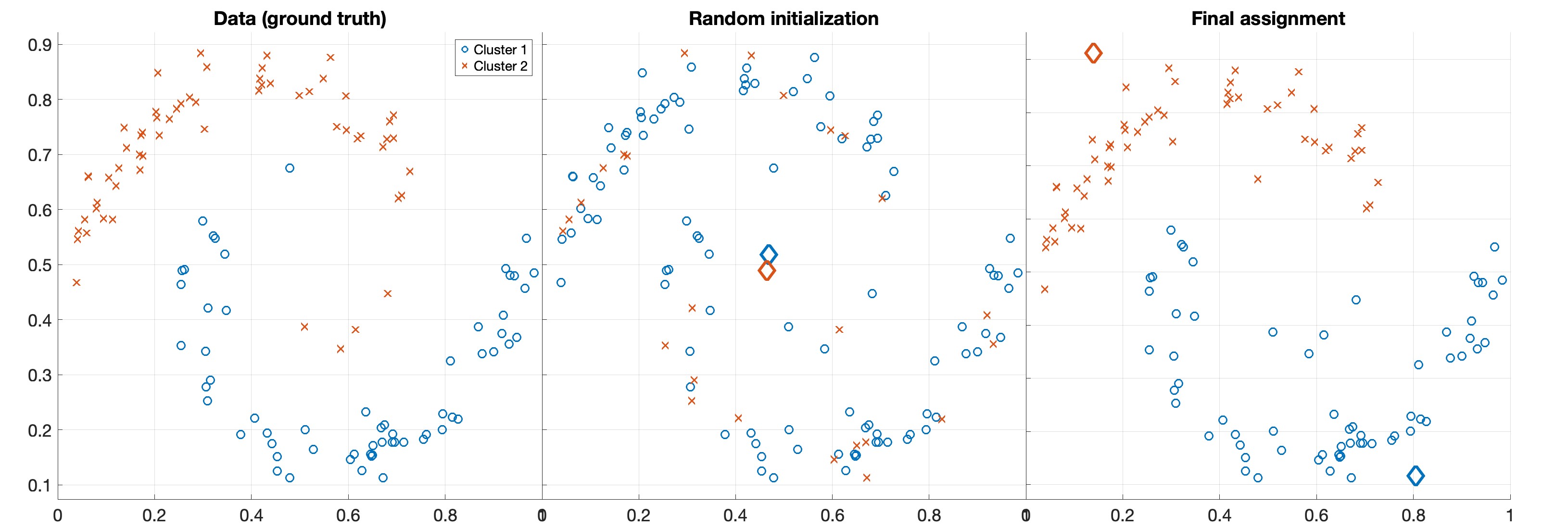}
\caption{Three-panel illustration of the Two-Spirals dataset.
Left: the original data points, colored by ground-truth labels.
Center: random initialization with $k=2$ clusters.
Right: final assignment obtained by T-ARC using the persistence-based similarity matrix.}
  \label{fig:ex_spiral}
\end{figure}

\begin{figure}[h!]
  \centering
  \begin{subfigure}{0.49\linewidth}
    \centering
    \includegraphics[width=\linewidth]{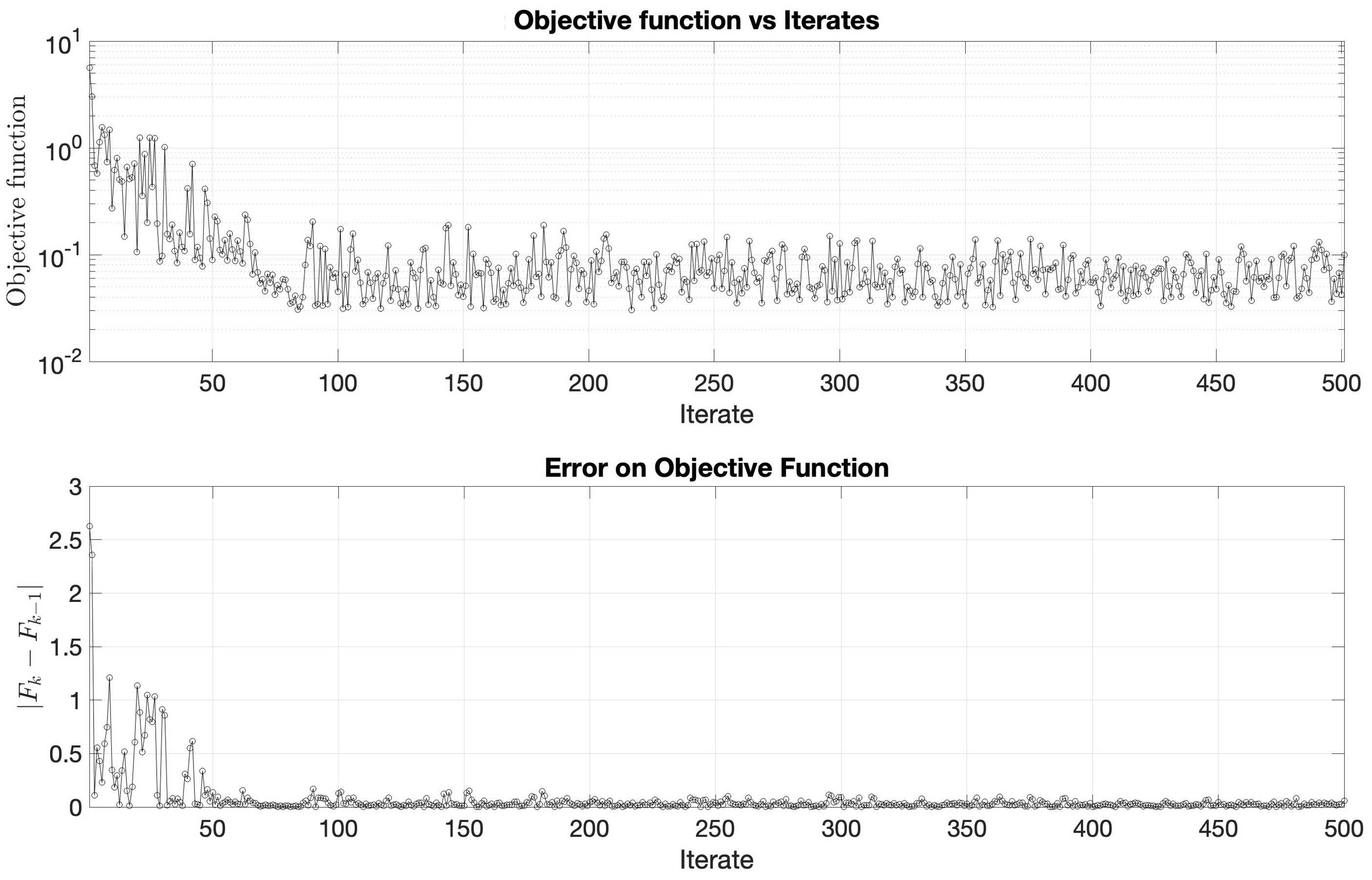}
    \caption{}
    \label{fig:ex_spiral_objective}
  \end{subfigure}
  \hfill
  \begin{subfigure}{0.49\linewidth}
    \centering
    \includegraphics[width=\linewidth]{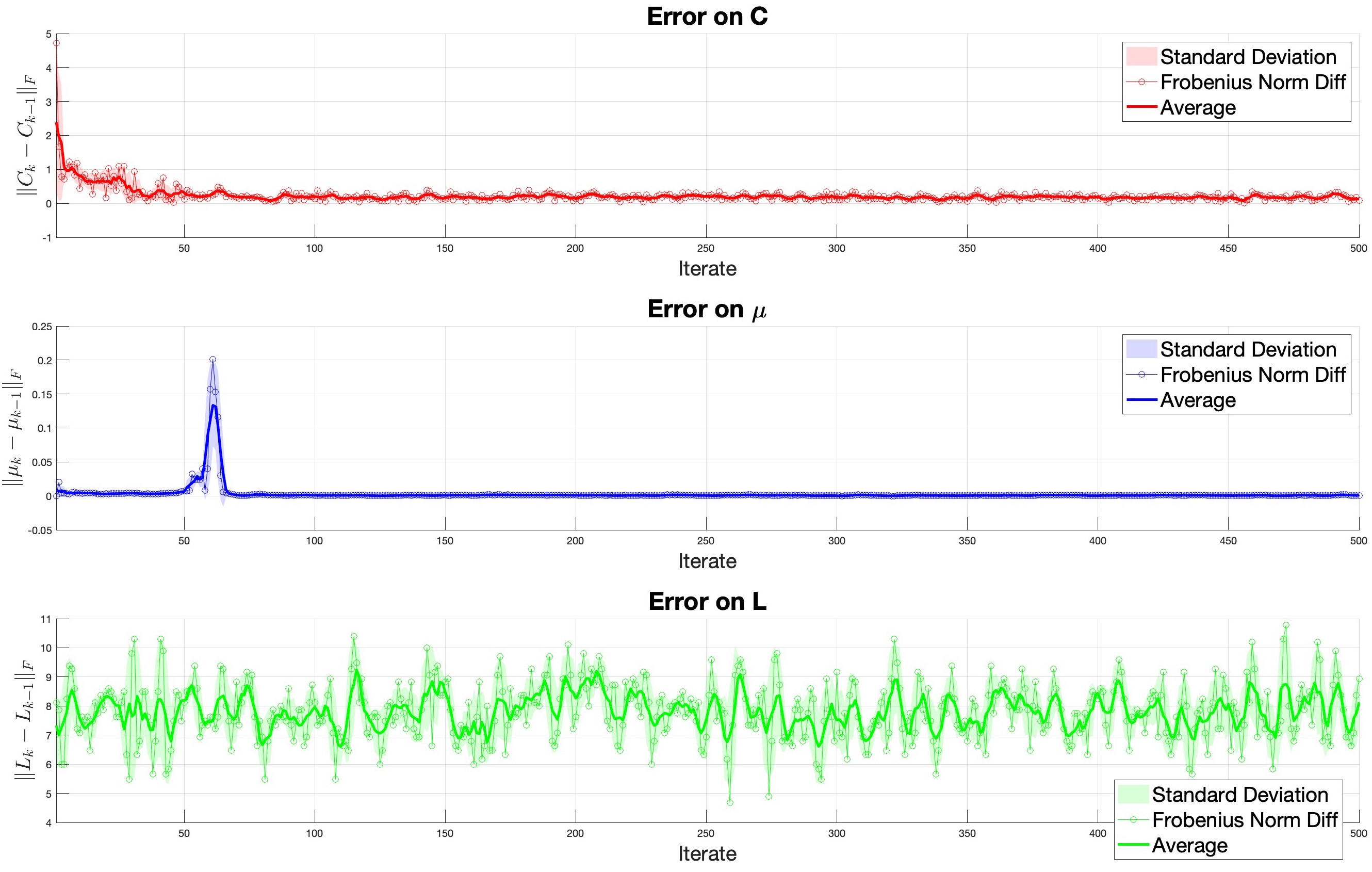}
    \caption{}
    \label{fig:ex_spiral_error}
  \end{subfigure}
  \caption{(a) Objective function value during iterations and (b) Frobenius norm error of $\bm{C}$, $\bm{\mu}$, and $\bm{L}$ during iterations for T-ARC with persistence-based similarity matrix on the Two-Spirals dataset.}
  \label{fig:example_spiral}
\end{figure}

Interleaved spirals are a classical benchmark for clustering methods, as the two clusters are not linearly separable and each arm winds around the other, so that points belonging to different clusters can be closer in Euclidean distance than points at opposite ends of the same arm. Centroid-based methods are therefore structurally unable to recover the correct partition. On this configuration, T-ARC (persistence) achieves the highest Accuracy and F1-score ($0.96$), clearly outperforming K-means and Spherical K-means (both $0.84$), its Euclidean-based variant ($0.85$), and Spectral clustering ($0.78$ and $0.76$, respectively). The persistence-based similarity encodes the connectivity of each arm at the scale at which the two spirals are still separate connected components, and the randomized Laplacian guided by this similarity allows T-ARC to follow the winding structure of the data. Spectral clustering is more exposed to spurious connections between adjacent arms, particularly near the center of the spirals, where the arms are closest.

As in the Two-Moons case, the Silhouette score shows the opposite ordering: K-means, Spherical K-means and T-ARC (Euclidean) attain the highest values ($0.68$), while T-ARC (persistence) ($0.58$) and Spectral clustering ($0.45$) lag behind. This is expected, since the Silhouette coefficient measures cluster compactness in Euclidean space and therefore rewards the convex partitions produced by centroid-based methods, which here cut across both spirals. The partition found by T-ARC (persistence) follows the arms and is topologically correct, but it is necessarily less compact in the Euclidean sense, which the Silhouette coefficient penalizes.

\paragraph{Example 4: Two-Eights Dataset}\label{ex:ex3}
We generated $160$ two-dimensional points arranged as four circles (with $40$ points per cluster) of radius $r=0.12$, grouped into two vertically-aligned pairs forming two ``figure-eight'' shapes. Circle centers are $(0.35,0.62)$, $(0.35,0.38)$, $(0.70,0.62)$, $(0.70,0.38)$, with Gaussian noise $\sigma=0.01$. The algorithm is initialized with $k=4$ (the number of circles) to probe whether T-ARC can identify higher-level topological structure. Fig.~\ref{fig:example3} shows the dataset, initialization, and final assignment.
\begin{figure}[h!]
        \centering
        \includegraphics[width=\linewidth]{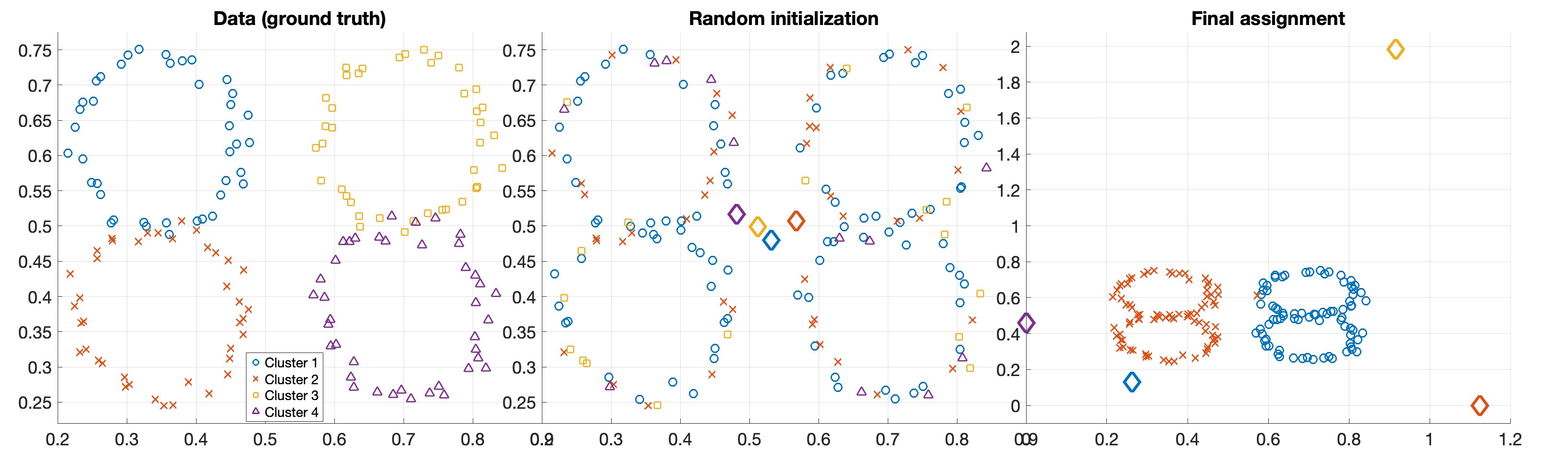}
  \caption{Three-panel illustration of the Two-Eights dataset.
Left: the original data points, colored by ground-truth labels.
Center: random initialization with $k=4$ clusters.
Right: final assignment obtained by T-ARC using the persistence-based similarity matrix; the algorithm collapses the four initial clusters into two macro-clusters, each corresponding to one of the two figure-eight shapes.}
  \label{fig:example3}
\end{figure}
\begin{figure}[h!]
  \centering
  \begin{subfigure}{0.49\linewidth}
    \centering
    \includegraphics[width=\linewidth]{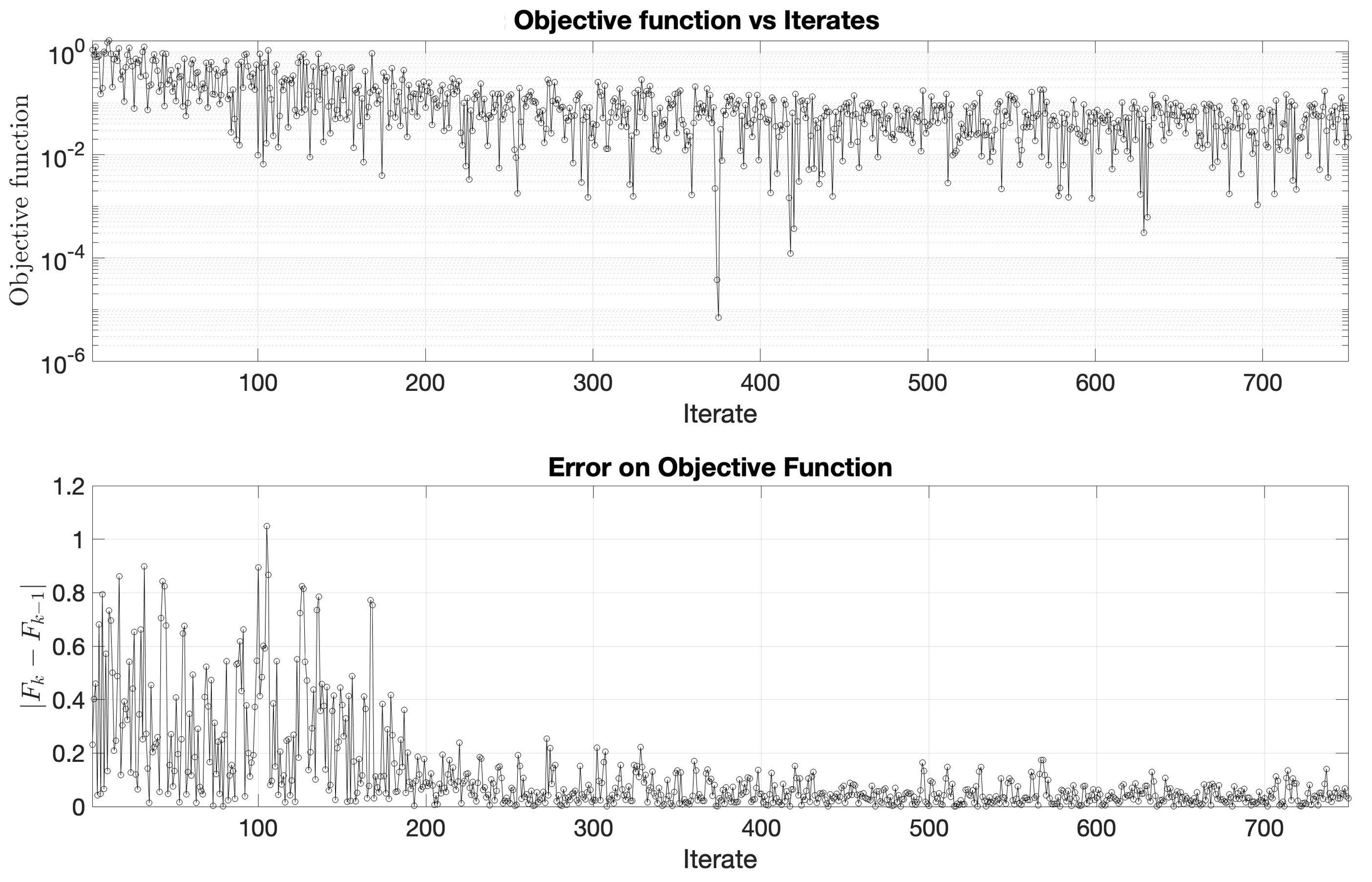}
    \caption{}
    \label{fig:ex3_objective}
  \end{subfigure}
  \hfill
  \begin{subfigure}{0.49\linewidth}
    \centering
    \includegraphics[width=\linewidth]{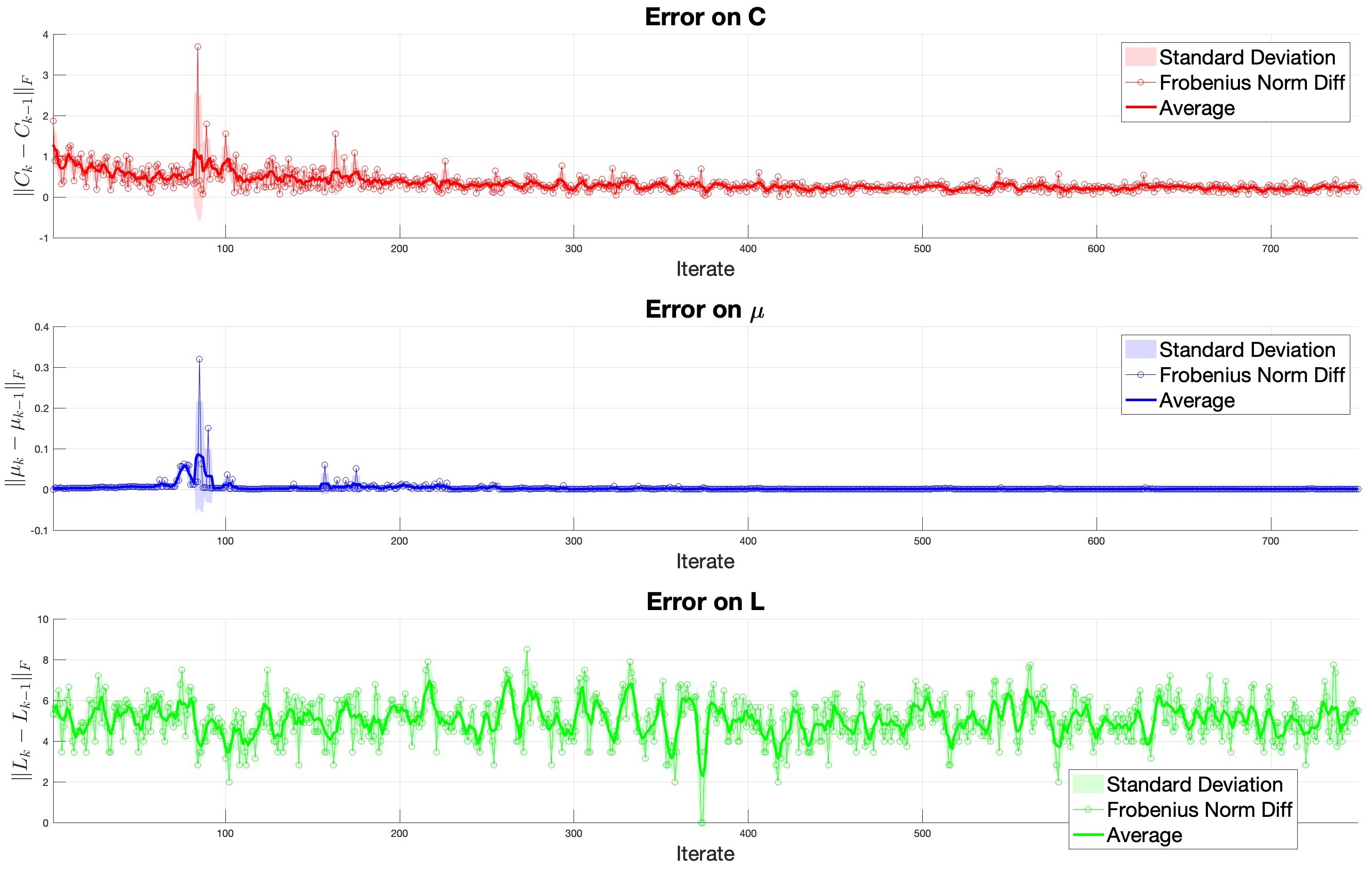}
    \caption{}
    \label{fig:ex3_error}
  \end{subfigure}
  \caption{(a) Objective function value during iterations and (b) Frobenius norm error of $\bm{C}$, $\bm{\mu}$, and $\bm{L}$ during iterations for T-ARC with persistence-based similarity matrix on Example 4.}
\label{fig:example3_graph}
\end{figure}
The external indices in Table~\ref{tab:results} may initially seem unfavorable for T-ARC (persistence): Accuracy $0.50$ and F1-score $0.33$ against K-means ($0.88$) and Spectral clustering ($0.87$). However, this comparison is misleading, and the behavior of T-ARC on this dataset is in fact particularly interesting:
\begin{itemize}[leftmargin=*]\setlength{\itemsep}{0pt}
\item Although initialized with $k=4$ clusters, T-ARC (persistence) converges to a partition into two groups, each corresponding to one of the two figure-eight shapes, rather than to the four individual circles. This partition matches the connected components of the data at the scale at which the two eights are separated, and is thus consistent with their intrinsic topological structure.
\item This result stems from the topology-aware structure of T-ARC: since the persistence-based prior enters the optimization through the graph-cut term, the algorithm favors partitions that are consistent with the connected components of the data at the relevant filtration scale, and therefore identifies the two figure-eight shapes rather than the four individual circles.
\item The Silhouette score supports this interpretation since the two figure-eight shapes are well separated: the topology-driven partition is also geometrically compact, and the two-cluster partition of T-ARC (persistence) attains $0.64$, against $0.56$ and $0.57$ for the four-cluster partitions of K-means and Spectral clustering, although partitions with different numbers of clusters are not directly comparable.
\end{itemize}
This result highlights the ability of the proposed approach to capture the underlying manifold structure rather than forcing a purely geometric partition of the data. Consequently, even though standard classification metrics may appear less favorable, the method successfully uncovers the topological organization of the dataset, demonstrating its robustness in settings where the true structure is governed by nonlinear relationships. Notably, T-ARC's behavior reflects a different strength: it identifies the higher-level topological structure (the two ``eights'') rather than the individual circles.

\paragraph{Example 5: Chain-link Dataset}\label{ex:ex4}

We generated a three-dimensional point cloud of $160$ points forming two circular structures (rings) embedded in $\mathbb{R}^3$, arranged in a \emph{chain-link configuration}. The first ring lies in the $XY$-plane and is centered at the origin, while the second lies in the $XZ$-plane and is shifted so that the two structures touch at a single point. The rings exhibit different sampling densities and noise levels: the first is densely sampled with low Gaussian noise ($\sigma = 0.03$), whereas the second is sparser and affected by higher noise ($\sigma = 0.12$). This dataset poses a significant challenge for clustering algorithms, as the two structures are nonlinearly separable and intersect at a point, while the imbalance in density and noise further increases the difficulty of the task.

On this dataset, the Euclidean-based variant of T-ARC achieves the best external indices, with an Accuracy and F1-score of $0.91$, outperforming Spectral clustering ($0.84$), K-means and T-ARC (persistence) (both $0.83$), and Spherical K-means ($0.51$). This is a notable departure from the other datasets: the persistence-based prior, which excels when the topological structure dominates the geometry, is here slightly less discriminative than the Euclidean similarity, likely because the two rings intersect at a single point and the persistence prior merges them into a single connected component at the relevant scale. In terms of internal validity, the results are more nuanced. The Silhouette scores are relatively close across methods, with T-ARC (persistence) and K-means tied at the top ($0.52$), closely followed by Spectral clustering ($0.51$), while T-ARC (Euclidean) attains $0.41$. This behavior is consistent with the geometric nature of the Silhouette coefficient, which favors convex and well-separated clusters even when they do not reflect the underlying topology.

Overall, the experiment highlights a trade-off between geometric compactness and topological correctness. Interestingly, the two T-ARC variants play complementary roles on this dataset: the Euclidean variant recovers the two rings more accurately, while the persistence variant, whose scores coincide with those of K-means, yields a more compact but less faithful partition, since the persistence prior merges the two rings at their contact point. This suggests that the Euclidean construction is preferable in the presence of intersecting structures, consistently with the overlapping Two sets case.

\begin{figure}[h!]
\includegraphics[width=\linewidth]{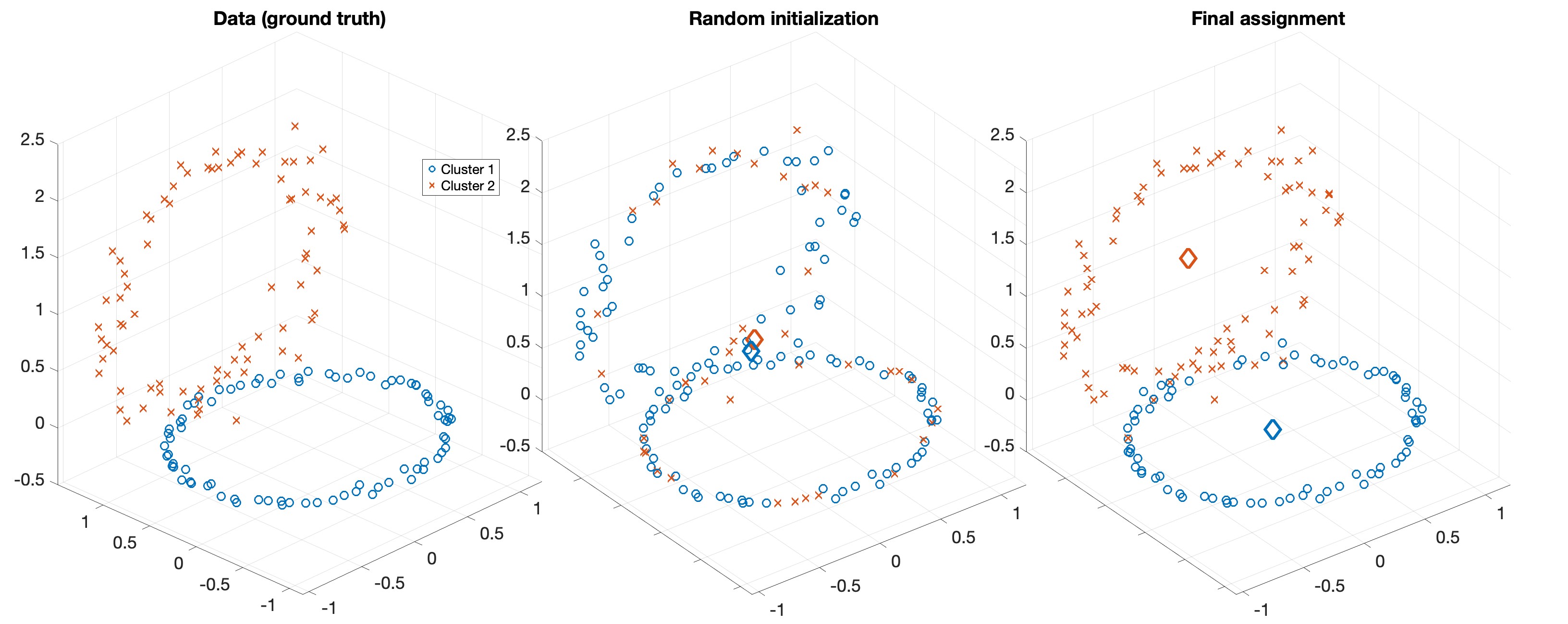}
\caption{Three-panel illustration of the Chain-link dataset.
Left: the original data points, colored by ground-truth labels.
Center: random initialization with $k=4$ clusters.
Right: final assignment obtained by T-ARC using the euclidean-based similarity matrix.}
  \label{fig:example4}
\end{figure}

\begin{figure}[h!]
  \centering
  \begin{subfigure}{0.49\linewidth}
    \centering
    \includegraphics[width=\linewidth]{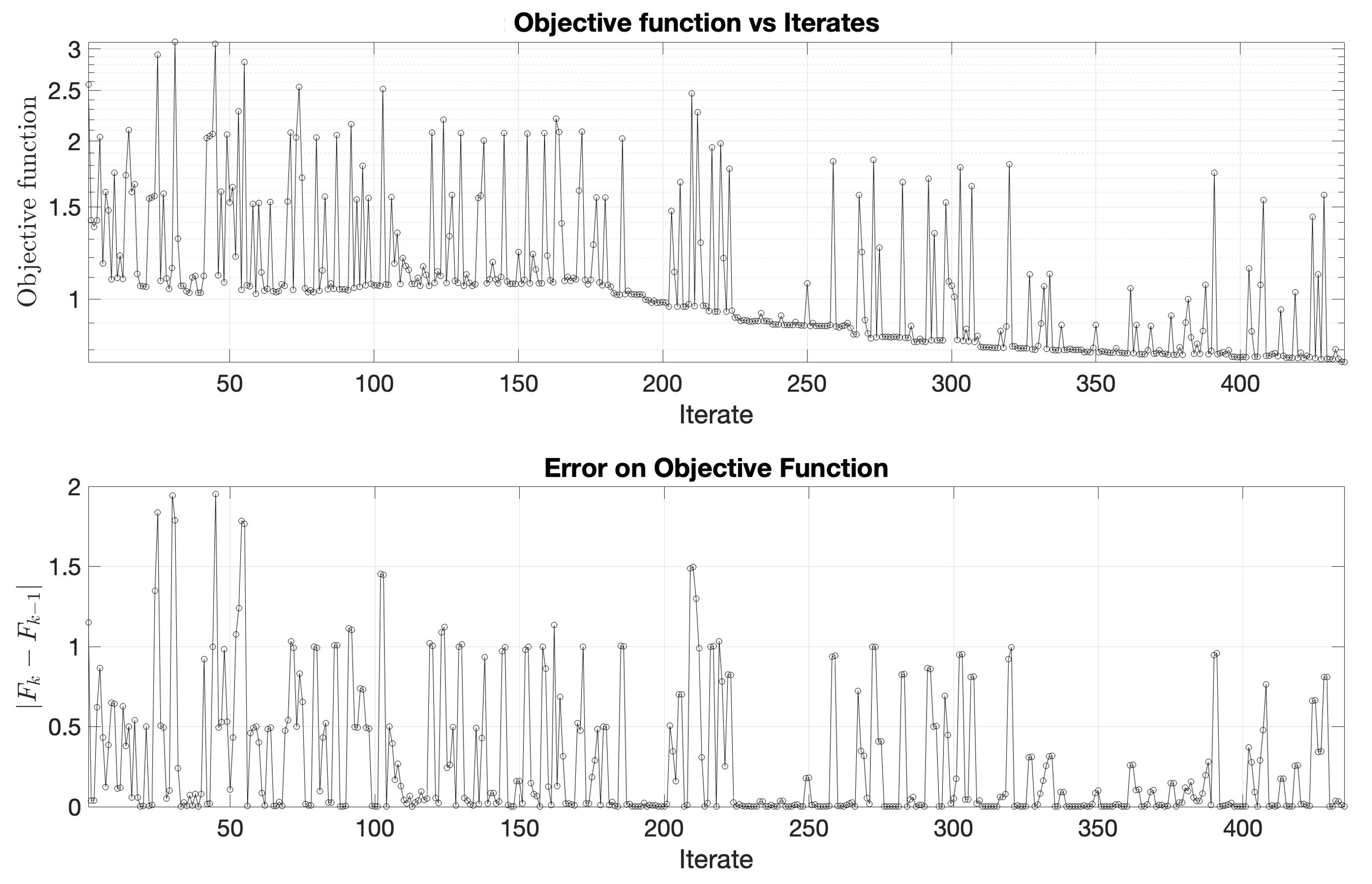}
    \caption{}
    \label{fig:ex4_objective}
  \end{subfigure}
  \begin{subfigure}{0.49\linewidth}
    \centering
    \includegraphics[width=\linewidth]{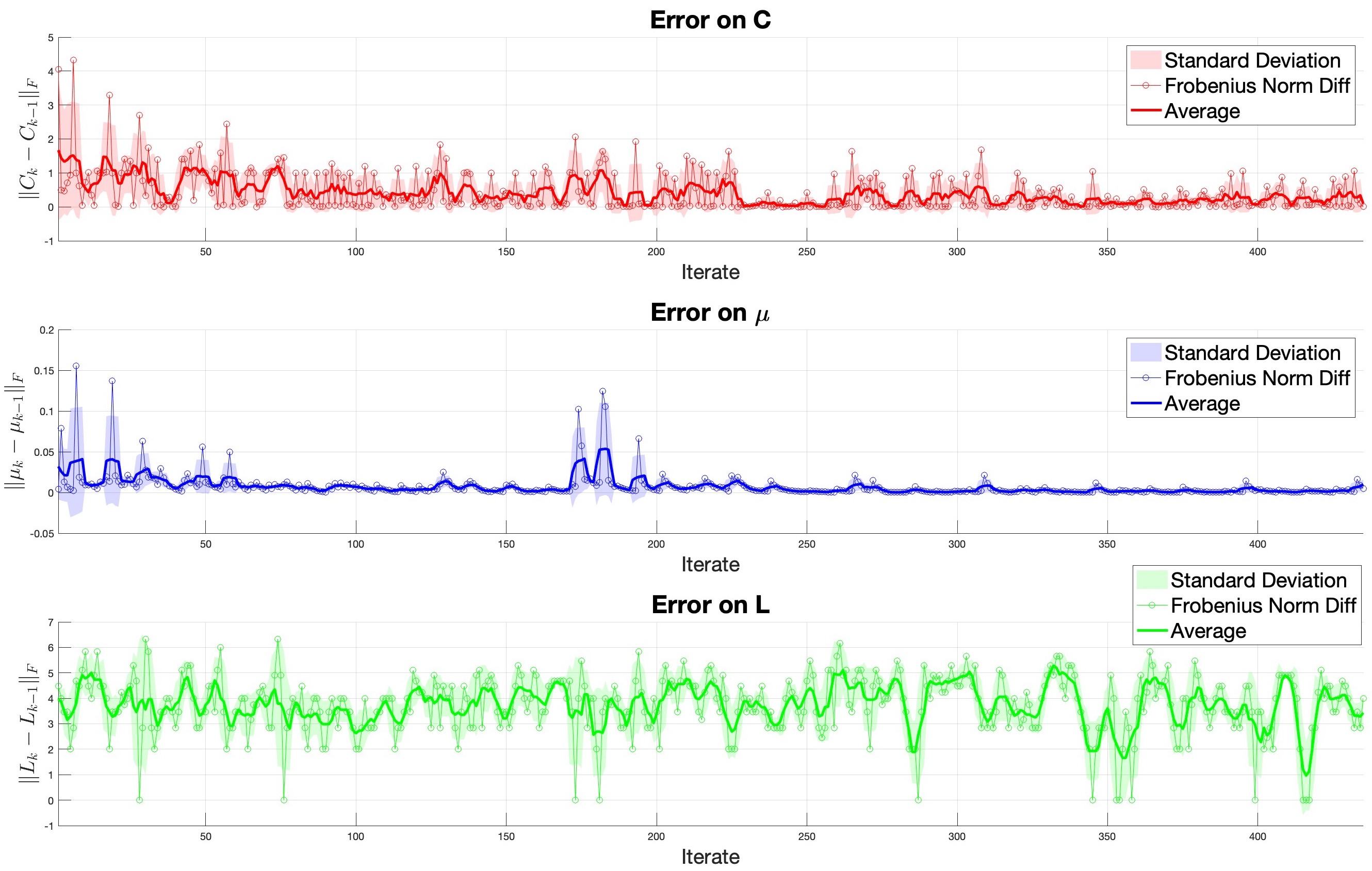}
    \caption{}
    \label{fig:ex4_error}
  \end{subfigure}
  \caption{(a) Objective function value during iterations and (b) Frobenius norm error of $\bm{C}$, $\bm{\mu}$, and $\bm{L}$ during iterations for T-ARC with euclidean-based similarity matrix on Example 5.}
  \label{fig:ex4}
\end{figure}

\paragraph{Example 6: Fashion-MNIST Subset}\label{ex:ex5}

As a real-world application, we consider subsets of the \textit{Fashion-MNIST} \cite{fashion_MNIST_dataset} dataset. We draw $10$ random subsets, each made of $50$ samples from each of three Fashion-MNIST classes---\emph{Trouser}, \emph{Sneaker}, and \emph{Bag}---so that every subset contains $150$ grayscale $28\times28$ images, each represented as a vector in $\mathbb{R}^{784}$ with pixel intensities rescaled to $[0,1]$. All methods are run on each subset from the same initialization, and we report mean $\pm$ standard deviation over the $10$ subsets.

T-ARC (persistence) attains an Accuracy of $0.91 \pm 0.02$ and an F1-score of $0.91 \pm 0.02$, and its Euclidean-based variant $0.91 \pm 0.03$ and $0.91 \pm 0.03$; the persistence-based prior matches or outperforms the Euclidean one in $9$ of the $10$ subsets (Table~\ref{tab:fashion_subsets}). Both variants outperform Spectral clustering ($0.79 \pm 0.03$ Accuracy, $0.76 \pm 0.05$ F1-score) and, on average, K-means ($0.89 \pm 0.13$ Accuracy, $0.88 \pm 0.14$ F1-score). The comparison with K-means is best read in terms of stability: started from the same initialization, K-means occasionally converges to poor local minima (down to $0.53$ Accuracy on one subset), whereas T-ARC never does, with a standard deviation about five to six times smaller. Spherical K-means attains the best external indices ($0.95 \pm 0.04$ Accuracy, $0.95 \pm 0.04$ F1-score), consistently with the angular structure of normalized pixel intensities. On the internal side, T-ARC (persistence) achieves the highest Silhouette score among all methods ($0.49 \pm 0.02$, versus $0.47 \pm 0.03$ for Spherical K-means and $0.46 \pm 0.08$ for K-means), indicating that its partitions are the most compact ones in the Euclidean sense.

Overall, this experiment shows that, on real high-dimensional data, T-ARC is competitive with classical methods, markedly more stable than K-means under the same initialization, and yields the most compact partitions, while Spherical K-means remains the strongest method in terms of label agreement on this dataset.

\begin{table}[h!]
\centering
\caption{Performance metrics of T-ARC and baseline methods on each of the $10$ random Fashion-MNIST subsets ($50$ images per class). Best value per metric and subset in \textbf{bold}; the last column reports mean $\pm$ standard deviation over the subsets.}
\label{tab:fashion_subsets}
\small
\setlength{\tabcolsep}{3pt}
\renewcommand{\arraystretch}{1.15}
\begin{adjustbox}{max width=\textwidth}
\begin{tabular}{@{}c l *{11}{c}@{}}
\toprule
 & \textbf{Method}
 & \multicolumn{10}{c}{\textbf{Subset}}
 & \textbf{Mean $\pm$ std} \\
\cmidrule(lr){3-12}
 & & 1 & 2 & 3 & 4 & 5 & 6 & 7 & 8 & 9 & 10 & \\
\midrule
\multirow{5}{*}{\textbf{Accuracy}}
 & T-ARC (persistence)  & 0.91 & 0.91 & 0.89 & 0.92 & 0.95 & 0.90 & 0.94 & 0.92 & 0.92 & 0.88 & $0.91 \pm 0.02$ \\
 & T-ARC (Euclidean)    & 0.91 & 0.89 & 0.89 & 0.93 & 0.95 & 0.89 & 0.93 & 0.92 & 0.90 & 0.88 & $0.91 \pm 0.03$ \\
 & K-means              & 0.93 & 0.53 & \textbf{0.92} & 0.95 & \textbf{0.97} & 0.87 & 0.93 & \textbf{0.95} & \textbf{0.95} & 0.87 & $0.89 \pm 0.13$ \\
 & Spherical K-means    & \textbf{0.96} & \textbf{0.99} & 0.87 & \textbf{0.99} & \textbf{0.97} & \textbf{0.94} & \textbf{0.95} & 0.91 & \textbf{0.95} & \textbf{0.93} & $\mathbf{0.95 \pm 0.04}$ \\
 & Spectral clustering  & 0.80 & 0.81 & 0.77 & 0.76 & 0.83 & 0.73 & 0.81 & 0.77 & 0.83 & 0.77 & $0.79 \pm 0.03$ \\
\midrule
\multirow{5}{*}{\textbf{F1-score}}
 & T-ARC (persistence)  & 0.91 & 0.91 & 0.88 & 0.92 & 0.95 & 0.90 & 0.94 & 0.92 & 0.92 & 0.88 & $0.91 \pm 0.02$ \\
 & T-ARC (Euclidean)    & 0.91 & 0.88 & 0.88 & 0.93 & 0.95 & 0.89 & 0.93 & 0.92 & 0.90 & 0.88 & $0.91 \pm 0.03$ \\
 & K-means              & 0.93 & 0.49 & \textbf{0.92} & 0.95 & \textbf{0.97} & 0.86 & 0.93 & \textbf{0.95} & \textbf{0.95} & 0.87 & $0.88 \pm 0.14$ \\
 & Spherical K-means    & \textbf{0.96} & \textbf{0.99} & 0.86 & \textbf{0.99} & \textbf{0.97} & \textbf{0.94} & \textbf{0.95} & 0.91 & \textbf{0.95} & \textbf{0.93} & $\mathbf{0.95 \pm 0.04}$ \\
 & Spectral clustering  & 0.78 & 0.79 & 0.75 & 0.73 & 0.83 & 0.68 & 0.80 & 0.73 & 0.82 & 0.75 & $0.76 \pm 0.05$ \\
\midrule
\multirow{5}{*}{\textbf{Silhouette}}
 & T-ARC (persistence)  & 0.45 & \textbf{0.50} & \textbf{0.50} & \textbf{0.50} & \textbf{0.51} & \textbf{0.46} & 0.53 & 0.48 & \textbf{0.49} & \textbf{0.47} & $\mathbf{0.49 \pm 0.02}$ \\
 & T-ARC (Euclidean)    & \textbf{0.49} & \textbf{0.50} & \textbf{0.50} & 0.49 & 0.50 & \textbf{0.46} & 0.52 & 0.49 & 0.41 & 0.40 & $0.48 \pm 0.04$ \\
 & K-means              & 0.48 & 0.23 & 0.49 & \textbf{0.50} & \textbf{0.51} & 0.45 & \textbf{0.54} & 0.49 & 0.48 & \textbf{0.47} & $0.46 \pm 0.08$ \\
 & Spherical K-means    & 0.46 & 0.45 & 0.49 & 0.48 & \textbf{0.51} & 0.41 & 0.52 & \textbf{0.50} & 0.46 & 0.44 & $0.47 \pm 0.03$ \\
 & Spectral clustering  & 0.33 & 0.36 & 0.33 & 0.32 & 0.36 & 0.17 & 0.31 & 0.31 & 0.33 & 0.32 & $0.31 \pm 0.05$ \\
\bottomrule
\end{tabular}
\end{adjustbox}
\end{table}

\begin{table}[h!]
\centering
\caption{Performance metrics of T-ARC and baseline methods across all experiments.
  Best values per metric and dataset are in \textbf{bold}.
  $\dagger$ T-ARC (persistence) converges to $k=2$ macro-clusters
  (one per eight); external indices computed w.r.t.\ the $k=4$ ground truth.
  Fashion-MNIST values are mean $\pm$ standard deviation over the $10$ subsets of Table~\ref{tab:fashion_subsets}.}
\label{tab:results}
\small
\setlength{\tabcolsep}{3pt}
\renewcommand{\arraystretch}{1.15}
\ifdefined\TeXspan\let\span\TeXspan\fi
\begin{adjustbox}{max width=\textwidth}
\begin{tabular}{@{}c l *{9}{c}@{}}
\toprule
 & \textbf{Method}
 & \multicolumn{3}{c}{\textbf{Two sets}}
 & \textbf{Two moons}
 & \textbf{Two spirals}
 & \textbf{Two eights$^\dagger$}
 & \textbf{Chain-link}
 & \textbf{Fashion-MNIST} \\
\cmidrule(lr){3-5}
 & & Well-separated & Closer & Overlapping
 &  & 
 & & & \\
\midrule
\multirow{5}{*}{\textbf{Accuracy}}
 & T-ARC (persistence)     & \textbf{1.00} & \textbf{0.99} & 0.83          & \textbf{1.00} & \textbf{0.96} & 0.50          & 0.83          & $0.91 \pm 0.02$ \\
 & T-ARC (Euclidean)       & \textbf{1.00} & 0.96          & 0.92          & 0.87          & 0.85          & 0.60          & \textbf{0.91} & $0.91 \pm 0.03$ \\
 & K-means                 & \textbf{1.00} & \textbf{0.99} & 0.93          & 0.89          & 0.84          & \textbf{0.88} & 0.83          & $0.89 \pm 0.13$ \\
 & Spherical K-means       & 0.51          & 0.51          & 0.51          & 0.87          & 0.84          & 0.53          & 0.51          & $\mathbf{0.95 \pm 0.04}$ \\
 & Spectral clustering     & \textbf{1.00} & \textbf{0.99} & \textbf{0.95} & \textbf{1.00} & 0.78          & 0.87          & 0.84          & $0.79 \pm 0.03$ \\
\midrule
\multirow{5}{*}{\textbf{F1-score}}
 & T-ARC (persistence)     & \textbf{1.00} & \textbf{0.99} & 0.83          & \textbf{1.00} & \textbf{0.96} & 0.33          & 0.83          & $0.91 \pm 0.02$ \\
 & T-ARC (Euclidean)       & \textbf{1.00} & 0.96          & 0.92          & 0.87          & 0.85          & 0.53          & \textbf{0.91} & $0.91 \pm 0.03$ \\
 & K-means                 & \textbf{1.00} & \textbf{0.99} & 0.93          & 0.89          & 0.84          & \textbf{0.88} & 0.83          & $0.88 \pm 0.14$ \\
 & Spherical K-means       & 0.51          & 0.51          & 0.51          & 0.87          & 0.84          & 0.55          & 0.51          & $\mathbf{0.95 \pm 0.04}$ \\
 & Spectral clustering     & \textbf{1.00} & \textbf{0.99} & \textbf{0.95} & \textbf{1.00} & 0.76          & 0.87          & 0.84          & $0.76 \pm 0.05$ \\
\midrule
\multirow{5}{*}{\textbf{Silhouette}}
 & T-ARC (persistence)     & \textbf{0.97} & \textbf{0.85} & 0.51          & 0.61          & 0.58          & \textbf{0.64} & \textbf{0.52} & $\mathbf{0.49 \pm 0.02}$ \\
 & T-ARC (Euclidean)       & \textbf{0.97} & 0.80          & 0.67          & 0.66          & \textbf{0.68} & 0.18          & 0.41          & $0.48 \pm 0.04$ \\
 & K-means                 & \textbf{0.97} & \textbf{0.85} & \textbf{0.69} & \textbf{0.67} & \textbf{0.68} & 0.56          & \textbf{0.52} & $0.46 \pm 0.08$ \\
 & Spherical K-means       & 0.02          & 0.13          & 0.24          & 0.66          & \textbf{0.68} & 0.18          & 0.40          & $0.47 \pm 0.03$ \\
 & Spectral clustering     & \textbf{0.97} & \textbf{0.85} & \textbf{0.69} & 0.61          & 0.45          & 0.57          & 0.51          & $0.31 \pm 0.05$ \\
\bottomrule
\end{tabular}
\end{adjustbox}
\end{table}
\section{Conclusion and Future Work} \label{sec:conclusion}

In this work, we introduced T-ARC, a novel clustering framework that integrates topological information directly into the optimization objective to overcome the geometric limitations of classical methods such as K-means. By complementing the standard data-fidelity term with a graph-cut regularization, T-ARC promotes cluster assignments that respect the connectivity structure of the data. The key methodological contribution lies in the joint learning of the graph structure and the cluster assignments: the latent graph is modeled as a random realization from an SBM, whose parameters are optimized within a DRO framework. This formulation yields closed-form updates, solved via a proximal point method, and ensures 
robustness to uncertainty in the graph distribution.

To further anchor the framework in topological principles, SBM and DRO are informed by a persistence-based similarity matrix constructed from the zero-dimensional persistent homology. This matrix translates the multiscale connectivity structure of the dataset into a pairwise similarity representation, which serves as a fixed topological prior throughout the optimization. The overall optimization proceeds via 
BCD, and we provided a convergence analysis showing that the expected value of a Lyapunov functional is nonincreasing along the iterations, including the stochastic graph updates, and therefore convergent, with vanishing successive differences of all blocks.

Experimental validation on synthetic datasets with complex geometries demonstrated that T-ARC successfully recovers the underlying topological structure where K-means fails, as in the Two-Moons and Two-Spirals configurations. Notably, on the two-eights dataset, T-ARC (with $k=4$) identified the two higher-level clusters corresponding to the two eights, a partition consistent with the connected components of the data at the relevant scale. On Fashion-MNIST, T-ARC was competitive with the baselines in terms of external validity, outperforming K-means and Spectral clustering on average while Spherical K-means remained the strongest method in terms of label agreement; at the same time, T-ARC proved markedly more stable than K-means under the same initialization and attained the highest average Silhouette score. The persistence-based similarity matrix matched or outperformed its Euclidean counterpart whenever the topological structure dominates the geometry (well-separated and closer Two sets, Two-Moons, Two-Spirals, Fashion-MNIST), while the Euclidean variant proved more robust when clusters overlap or intersect (overlapping Two sets, Chain-link), delineating the regime in which the topological prior is most informative.

Several directions for future research emerge from this work:
\begin{itemize}[leftmargin=*]\setlength{\itemsep}{0pt}
    \item Although we focused on the zero-dimensional persistent homology, the framework could naturally accommodate similarity matrices derived from higher-order persistent homology, potentially capturing loop- and void-like structures in the data.

     \item The computation of the persistence-based similarity matrix via Vietoris-Rips filtration becomes demanding for large point clouds; exploring approximate or landmark-based methods for persistent homology could extend the applicability of T-ARC to larger-scale datasets.

    \item Finally, the SBM represents only one among many possible probabilistic formulations for the latent graph; richer families of random graphs that account for degree heterogeneity or incorporate geometric structure could yield more flexible connectivity representations.
\end{itemize}

In summary, T-ARC represents a principled step toward embedding topological structure directly into clustering optimization, bridging topological data analysis and unsupervised learning.

{
\paragraph{Acknowledgment}
S.G.D., N.D.B., F.E., and L.S. are members of the Gruppo Nazionale Calcolo Scientifico - Istituto Nazionale di Alta Matematica (GNCS-INdAM). S.G.D. is funded by a PhD fellowship within the framework of the Italian ``D.M. n. 117, March 2, 2023'' - under the National Recovery and Resilience Plan, Msn. 4, Comp. 2, Investment 3.3 - PhD Project ``Topological Data Analysis and optimization for industrial processes'', co-supported by ``Pirelli Tyre S.p.A.'' (CUP H91I23000170007). This work was completed during a visiting research period of S.G.D. at the University of Southampton.\\
F.E. and L.S. are partially supported by "INdAM - GNCS Project" CUP E53C25002010001.

\paragraph{Author Contributions}
S.G.D.: Conceptualization, Methodology, Software, Validation, Formal analysis, Data Curation, Investigation, Writing - Original Draft, Writing - Review \& Editing, Visualization.  
A.A.: Conceptualization, Methodology, Formal analysis, Writing - Original Draft, Writing - Review \& Editing, Supervision.  
F.E. and L.S.: Conceptualization, Methodology, Writing - Review \& Editing.
N.D.B.: Conceptualization, Methodology, Writing - Review \& Editing, Supervision.

\bibliographystyle{siam}
\bibliography{bibliography}

@article{SelMATCOM,
title = {Topological data analysis for resilience assessment of water distribution networks},
journal = {Mathematics and Computers in Simulation},
volume = {231},
pages = {62-70},
year = {2025},
issn = {0378-4754},
doi = {https://doi.org/10.1016/j.matcom.2024.12.001},
author = {Laura Selicato and Alessandro Pagano and Flavia Esposito and Matteo Icardi}
}

@misc{fashion_MNIST_dataset,
      title={Fashion-MNIST: a Novel Image Dataset for Benchmarking Machine Learning Algorithms}, 
      author={Han Xiao and Kashif Rasul and Roland Vollgraf},
      year={2017},
      eprint={1708.07747},
      archivePrefix={arXiv},
      primaryClass={cs.LG},
      url={https://arxiv.org/abs/1708.07747}, 
}

@book{gudhi
, title        = "{GUDHI} User and Reference Manual"
, author      = "{The GUDHI Project}"
, publisher     = "{GUDHI Editorial Board}"
, year         = 2015
, url =    "http://gudhi.gforge.inria.fr/doc/latest/"
}

@article{SBM,
  title = {A review of stochastic block models and extensions for graph clustering},
  volume = {4},
  ISSN = {2364-8228},
  url = {http://dx.doi.org/10.1007/s41109-019-0232-2},
  DOI = {10.1007/s41109-019-0232-2},
  number = {1},
  journal = {Applied Network Science},
  publisher = {Springer Science and Business Media LLC},
  author = {Lee,  Clement and Wilkinson,  Darren J.},
  year = {2019},
  month = dec 
}

@article{DRO,
  title = {Distributionally robust optimization},
  volume = {34},
  ISSN = {1474-0508},
  url = {http://dx.doi.org/10.1017/S0962492924000084},
  DOI = {10.1017/s0962492924000084},
  journal = {Acta Numerica},
  publisher = {Cambridge University Press (CUP)},
  author = {Kuhn,  Daniel and Shafiee,  Soroosh and Wiesemann,  Wolfram},
  year = {2025},
  month = jul,
  pages = {579-804}
}

@article{proximal_point_method,
  title = {Proximal Algorithms},
  volume = {1},
  ISSN = {2167-3918},
  url = {http://dx.doi.org/10.1561/2400000003},
  DOI = {10.1561/2400000003},
  number = {3},
  journal = {Foundations and Trends in Optimization},
  publisher = {Emerald},
  author = {Parikh,  Neal and Boyd,  Stephen},
  year = {2014},
  month = jan,
  pages = {127-239}
}

@article{Lyapunov,
  title = {A Lyapunov-type approach to convergence of the Douglas-Rachford algorithm for a nonconvex setting},
  volume = {73},
  ISSN = {1573-2916},
  url = {http://dx.doi.org/10.1007/s10898-018-0677-3},
  DOI = {10.1007/s10898-018-0677-3},
  number = {1},
  journal = {Journal of Global Optimization},
  publisher = {Springer Science and Business Media LLC},
  author = {Dao,  Minh N. and Tam,  Matthew K.},
  year = {2018},
  month = jun,
  pages = {83-112}
}

@book{TDA,
  title = {Algebraic Foundations for Applied Topology and Data Analysis},
  journal = {Mathematics of Data},
  publisher = {Springer International Publishing},
  author = {Schenck,  Hal},
  year = {2022}
}

@article{lloyd1982,
  author  = {Lloyd, Stuart P.},
  title   = {Least Squares Quantization in {PCM}},
  journal = {IEEE Transactions on Information Theory},
  volume  = {28},
  number  = {2},
  pages   = {129--137},
  year    = {1982},
  doi     = {10.1109/TIT.1982.1056489}
}

@article{jain2010,
  author  = {Jain, Anil K.},
  title   = {Data Clustering: 50 Years Beyond {K}-Means},
  journal = {Pattern Recognition Letters},
  volume  = {31},
  number  = {8},
  pages   = {651--666},
  year    = {2010},
  doi     = {10.1016/j.patrec.2009.09.011}
}

@article{vonluxburg2007,
  author  = {von Luxburg, Ulrike},
  title   = {A Tutorial on Spectral Clustering},
  journal = {Statistics and Computing},
  volume  = {17},
  number  = {4},
  pages   = {395--416},
  year    = {2007},
  doi     = {10.1007/s11222-007-9033-z}
}

@article{carlsson2009,
  author  = {Carlsson, Gunnar},
  title   = {Topology and Data},
  journal = {Bulletin of the American Mathematical Society},
  volume  = {46},
  number  = {2},
  pages   = {255--308},
  year    = {2009},
  doi     = {10.1090/S0273-0979-09-01249-X}
}

@book{edelsbrunner2010,
  author    = {Edelsbrunner, Herbert and Harer, John L.},
  title     = {Computational Topology: An Introduction},
  publisher = {American Mathematical Society},
  address   = {Providence, RI},
  year      = {2010},
  doi       = {10.1090/mbk/069}
}

@article{zomorodian2005,
  author  = {Zomorodian, Afra and Carlsson, Gunnar},
  title   = {Computing Persistent Homology},
  journal = {Discrete \& Computational Geometry},
  volume  = {33},
  number  = {2},
  pages   = {249--274},
  year    = {2005},
  doi     = {10.1007/s00454-004-1146-y}
}

@article{cohen-steiner2007,
  author  = {Cohen-Steiner, David and Edelsbrunner, Herbert
             and Harer, John},
  title   = {Stability of Persistence Diagrams},
  journal = {Discrete \& Computational Geometry},
  volume  = {37},
  number  = {1},
  pages   = {103--120},
  year    = {2007},
  doi     = {10.1007/s00454-006-1276-5}
}

@article{hensel2021,
  author  = {Hensel, Felix and Moor, Michael and Rieck, Bastian},
  title   = {A Survey of Topological Machine Learning Methods},
  journal = {Frontiers in Artificial Intelligence},
  volume  = {4},
  pages   = {681108},
  year    = {2021},
  doi     = {10.3389/frai.2021.681108}
}

@article{BCD,
  author  = {Xu, Yangyang and Yin, Wotao},
  title   = {A Block Coordinate Descent Method for Regularized 
             Multiconvex Optimization with Applications to 
             Nonnegative Tensor Factorization and Completion},
  journal = {SIAM Journal on Imaging Sciences},
  volume  = {6},
  number  = {3},
  pages   = {1758--1789},
  year    = {2013},
  doi     = {10.1137/120887795}
}

@article{huang2025inhomogeneous,
  title={Inhomogeneous Graph Trend Filtering via A $l_\{$2, 0$\}$-norm Cardinality Penalty},
  author={Huang, Xiaoqing and Ang, Andersen and Huang, Kun and Zhang, Jie and Wang, Yijie},
  journal={IEEE Transactions on Signal and Information Processing over Networks},
  year={2025},
  publisher={IEEE}
}

@article{Rousseeuw1987,
  author = {Rousseeuw, Peter J.},
  title = {Silhouettes: A graphical aid to the interpretation and validation of cluster analysis},
  journal = {Journal of Computational and Applied Mathematics},
  year = {1987},
  volume = {20},
  pages = {53--65},
  doi = {10.1016/0377-0427(87)90125-7},
  month = {nov}
}

@ARTICLE{kmeans_limitations,
  author={Hong, Jiazhen and Qian, Wei and Chen, Yudong and Zhang, Yuqian},
  journal={IEEE Transactions on Knowledge and Data Engineering}, 
  title={A Geometric Approach to $k$k-Means Clustering}, 
  year={2026},
  volume={38},
  number={3},
  pages={1442-1453},
  doi={10.1109/TKDE.2025.3616858}}

@article{ding2010convex,
  title={Convex and semi-nonnegative matrix factorizations},
  author={Ding, Chris HQ and Li, Tao and Jordan, Michael I},
  journal={IEEE Transactions on Pattern Analysis and Machine Intelligence},
  volume={32},
  number={1},
  pages={45--55},
  year={2010},
  publisher={IEEE}
}

@article{cai2011graph,
  author  = {Cai, Deng and He, Xiaofei and Han, Jiawei and Huang, Thomas S.},
  title   = {Graph Regularized Nonnegative Matrix Factorization for Data Representation},
  journal = {IEEE Transactions on Pattern Analysis and Machine Intelligence},
  volume  = {33},
  number  = {8},
  pages   = {1548--1560},
  year    = {2011},
  doi     = {10.1109/TPAMI.2010.231}
}
}

\normalsize

\appendix

\section{Derivation of Feasibility Constraints on $b$}
\label{sec:appendix_constraints}

Under the smoothed parametric representation $\bm{B}(a,b) = a\bm{S} + b(\ones_n\ones_n^\top - \bm{S})$, the ambiguity set $\cB$ (equation~\eqref{set:probability_set_according_to_DRO}) imposes three constraints on $a$ and $b$. Throughout, we substitute $a = 1 + b(1-n)$ (derived from row-stochasticity) to obtain constraints solely in $b$.

\paragraph{Row-stochasticity.}
Imposing $\bm{B}(a,b)\ones_n = \ones_n$ and using $\bm{S}\ones_n = \ones_n$:
\[
  (a\bm{S} + b(\ones_n\ones_n^\top - \bm{S}))\ones_n = a\ones_n + b(n-1)\ones_n = \ones_n,
\]
yielding $a + b(n-1) = 1$, i.e., $a = 1 + b(1-n)$. Positivity $a > 0$ then gives
\begin{equation}\label{eq:b_bound_a_pos_app}
  b < \tfrac{1}{n-1}.
\end{equation}

\paragraph{Nonnegativity.}
The condition $B(a,b)_{ij} \geq 0$ requires $aS_{ij} + b(1-S_{ij}) \geq 0$ for all $i,j$. The most restrictive bound is attained at $M \coloneqq \max_{i,j} S_{ij}$, yielding $a \geq -b(1-M)/M$. Substituting $a = 1 + b(1-n)$ gives
\begin{equation}\label{eq:b_bound_M_app}
  b \leq \tfrac{1}{n - 1/M}.
\end{equation}
Since each row of $\bm{S}$ is normalized, $1/n \leq M \leq 1$, hence $0 \leq n - 1/M \leq n-1$, so $1/(n - 1/M) \geq 1/(n-1)$, and constraint \eqref{eq:b_bound_M_app} is always dominated by \eqref{eq:b_bound_a_pos_app}.

\paragraph{Distance constraint.}
The Frobenius constraint $\|\bm{B}(a,b) - \bm{S}\|_F \leq r$ simplifies (using $a = 1 + b(1-n)$) to $\|b(\ones_n\ones_n^\top - n\bm{S})\|_F \leq r$. Letting $S_2 \coloneqq \sum_{i,j} S_{ij}^2$ (which satisfies $S_2 \geq 1$ for row-stochastic $\bm{S}$ with $S_2 > 1$ generically), this gives $b \leq r/(n\sqrt{S_2-1})$.

\paragraph{Final bounds.}
Combining and discarding the dominated nonnegativity constraint:
\begin{equation}\label{eq:b_bounds_app}
  0 \leq b \leq b_{\max} 
  \coloneqq \min\!\left\{
    \frac{1}{n-1},\quad \frac{r}{n\sqrt{S_2-1}}
  \right\}.
\end{equation}

\end{document}